%% file: main.tex
\documentclass[11pt]{article} 
\usepackage[numbers]{natbib}
\usepackage[margin=1.25in]{geometry}

\usepackage{times}
\usepackage{graphicx}
\usepackage{subcaption}
\usepackage{amssymb,amsfonts,bm,url,dsfont,color,mathtools,stmaryrd}
\usepackage{float}
\usepackage{xspace} 

\def\E{\mathbb{E}}
\def\1{\mathbf{1}}
\def\P{\mathbb{P}}

\usepackage{amsthm}
\newtheorem{lemma}{Lemma}
\newtheorem{proposition}{Proposition}
\newtheorem{theorem}{Theorem}
\newtheorem{corollary}{Corollary}
\newtheorem{remark}{Remark}
\newtheorem{claim}{Claim}

\newtheorem{fact}{Fact}

\newtheorem{definition}{Definition}

\newcommand{\calK}{\mathcal{K}}

\newcommand{\calN}{\mathcal{N}}

\newcommand{\calS}{\mathcal{S}}

\newcommand{\calA}{\mathcal{A}}

\newcommand{\calE}{\mathcal{E}}

\newcommand{\TOP}{\mathrm{TOP}}

\newcommand{\calF}{\mathcal{F}}

\newcommand{\bS}{\mathbf{S}}

\newcommand{\Alg}{\mathsf{Alg}}
\newcommand{\MAB}{\mathsf{MAB}}
\newcommand{\TopK}{\mathsf{TopK}}

\newcommand{\Sim}{\mathsf{Sim}}

\newcommand{\Trs}{\mathsf{Tr}}
\newcommand{\Trshat}{\widehat{\mathsf{Tr}}}
\newcommand{\Xhat}{\widehat{X}}

\newcommand{\Alt}{\mathrm{Alt}}
\newcommand{\kl}{\mathrm{kl}}

\newcommand{\KL}{\mathrm{KL}}
\newcommand{\TV}{\mathrm{TV}}

\newcommand{\DeltaEff}{\Delta_{\mathrm{eff}}}

\newcommand{\N}{\mathbb{N}}

\newcommand{\R}{\mathbb{R}}
\newcommand{\I}{\mathbb{I}}
\newcommand{\Exp}{\mathbb{E}}
\newcommand{\Q}{\mathbb{Q}}

\renewcommand{\P}{\mathbf{P}}

\renewcommand{\Pr}{\mathbb{P}}

\usepackage[vlined,ruled,linesnumbered]{algorithm2e}

\title{The Simulator: Understanding Adaptive Sampling in the Moderate-Confidence Regime}

\author{
  Max Simchowitz\\
  \texttt{msimchow@berkeley.edu}
  \and
  Kevin Jamieson\\
  \texttt{kjamieson@berkeley.edu}
  \and 
  Benjamin Recht
  \\
  \texttt{brecht@berkeley.edu}
}
\begin{document}

\maketitle
\input{abstract.tex}

\input{intro_colt3.tex}

\input{prelim_colt.tex}
\section{Statements of Lower Bound Results} \label{sec:lower_bounds}
\input{oracle_colt.tex}

\input{multi_hyp_contributions_colt.tex}
\input{upper_bounds.tex}
\input{simulator.tex}

\input{two_hyp.tex}


\input{conclusion.tex}

\newpage




\newpage
\small
\bibliographystyle{IEEEtranN}
\bibliography{yaba}
\normalsize
\newpage
\appendix

\input{upper_bound_proof.tex}
\input{permutation_upper_bounds.tex}

\input{two_hypothesis_proofs.tex}

\input{multi_hypothesis_proofs.tex}

\section{Lower Bounds for Distinct Measures\label{sec:distinct_measure_lb}}
\input{tiltings_explanation.tex}
\input{distinct_measures.tex}
\input{distinct_measures_2.tex}

\input{proof_of_topk.tex}


\end{document}

%% file: abstract.tex

\begin{abstract}
We propose a novel technique for analyzing adaptive sampling called the {\em Simulator}. Our approach differs from the existing methods by considering not how much information could be gathered by any fixed sampling strategy, but how difficult it is to distinguish a good sampling strategy from a bad one given the limited amount of data collected up to any given time. This change of perspective allows us to match the strength of both Fano and change-of-measure techniques, without succumbing to the limitations of either method. For concreteness, we apply our techniques to a structured multi-arm bandit problem in the fixed-confidence pure exploration setting, where we show that the constraints on the means imply a substantial gap between the moderate-confidence sample complexity, and the asymptotic sample complexity as $\delta \to 0$ found in the literature. We also prove the first instance-based lower bounds for the top-k problem which incorporate the appropriate log-factors. Moreover, our lower bounds zero-in on the number of times each \emph{individual} arm needs to be pulled, uncovering new phenomena which are drowned out in the aggregate sample complexity. Our new analysis inspires a simple and near-optimal algorithm for the best-arm and top-k identification, the first {\em practical} algorithm of its kind for the latter problem which removes extraneous log factors, and outperforms the state-of-the-art in experiments. 
\end{abstract}

%% file: intro_colt3.tex

\section{Introduction}

The goal of adaptive sampling is to estimate some unknown property $S^*$ about the world, using as few measurements from a set of possible measurement actions $[n] = \{1,\dots,n\}$\footnote{We only work with finitely many measurement actions, but this may be generalized as in Arias-Castro et al.~\cite{arias2013fundamental}}. At each time step $t = 1,2,\dots$, a learner chooses a measurement action $a_t \in [n]$ based on past observations, and receives an observation $X_{a_t,t} \in \R$. We assume that the observations are drawn i.i.d from a distribution $\nu_a$ over $\R$, which is unknown to the learner. In particular, the vector of distributions $\nu = (\nu_1,\dots,\nu_n)$, called the \emph{instance}, encodes the distribution of all possible measurement actions. The instance $\nu$ can be thought of as describing the state of the world, and that our property of interest $S^* = S^*(\nu)$ is a function of  the instance. We focus on what is called the \emph{fixed-confidence pure-exploration} setting, where the algorithm decides to stop at some (possibly random) time $T$, and returns an output $\widehat{S}$ which is allowed to differ from $S^*(\nu)$ with probability at most $\delta$ on any instance $\nu$. Since $T$ is exactly equal to the number of measurements taken, the goal of adaptive pure-exploration problems is to design algorithms for which $T$ is as small as possible, either in expectation or with high probability. 

Crucially, we often expect the instance $\nu$ to lie in a known constraining set $\calS$. This allows us to encode a broad range of problems of interest as pure-exploration multi-arm bandit ($\MAB$) problems \citep{bechhofer1958sequential,even2006action} with structural constraints. As an example, the adaptive linear prediction problem of \citep{soare2014best,2016arXiv161004491L} (known in the literature as \emph{linear bandits}), is equivalent to $\MAB$, subject to the constraint that the mean vector $\mu = (\mu_1,\dots,\mu_n)$ (where $\mu_a := \Exp_{X_a \sim \nu_a}[X_a]$) lies in the subspace spanned by the rows of $V = \begin{bmatrix} v_1 & | &  v_2 & | \dots | & v_n\end{bmatrix}$, where $v_1,\dots,v_n \in \R^d$ are the vector-valued features associated with arms $1$ through $n$. The noisy combinatorial optimization problems of \cite{yue2011linear,simchowitz2016best,gopalan2014thompson} can be also be cast in this fashion. Moreover, by considering properties $S^*(\nu)$ other than the top mean, one can use the above framework to model signal recovery and compressed sensing \citep{arias2013fundamental,castro2014adaptive}, subset-selection \citep{kalyanakrishnan2012pac}, and additional variants of combinatorial optimization \citep{chen2014combinatorial,chen2016pure,kveton2014matroid}.





The  purpose of this paper is to present new machinery to better understand the consequences of structural constraints $\calS$, and types of objectives $S^*(\nu)$ on the sample complexity of adaptive learning problems. This paper presents bounds for some structured adaptive sampling problems which characterize the sample complexity in the regime where the probability of error $\delta$ is a moderately small constant (e.g. $\delta = .05$, or even inverse-polynomial in the number of measurements). In contrast, prior work has addressed the sample complexity of adaptive samplings problems in the asymptotic regime that $\delta \to 0$, where such problems often admit algorithms whose asymptotic dependence on $\delta$ matches lower bounds \emph{for each ground-truth instance}, even matching the exact instance-dependent leading constant \citep{garivier2016optimal,russo2016simple,luedtke2016asymptotically}. Analogous asymptotically-sharp and instance-specific results (even for structured problems) also hold in the regret setting where the time horizon $T \to \infty$ \citep{lai1985asymptotically, gopalan2014thompson,magureanu2014lipschitz,combes2015combinatorial,talebi2016optimal}. 

The upper and lower bounds in this paper demonstrate that the $\delta \to 0$ asymptotics can paint a highly misleading picture of the true sample complexity when $\delta$ is not-too-small. This occurs for two reasons:
\begin{enumerate}
	\item Asymptotic characterizations of the sample complexity of adaptive estimation problems occur on a time horizon where the learner can learn an optimal measurement allocation tailored to the ground truth instance $\nu$. In the short run, however, learning favorable measurement allocations is extremely costly, and the learning good allocations requires considerably more samples to learn than it itself would prescribe.
	\item Asymptotic characterizations are governed by the complexity of discriminating the ground truth $\nu$ from any single, alternative hypothesis. This neglects multiple-hypothesis and suprema-of-empirical-process effects that are ubiquitous in high-dimensional statistics and learning theory (e.g. those reflected in Fano-style bounds).
\end{enumerate}
To understand these effects, we introduce a new framework for analyzing adaptive sampling called the ``Simulator''. Our approach differs from the existing methods by considering not how much information could be gathered by any fixed sampling strategy, but how difficult it is to distinguish a good sampling strategy from a bad one, given any limited amount of data collected up to any given time. Our framework allows us to characterize granular, instance dependent properties that any successful adaptive learning algorithm must have. In particular, these insights inspire a new, theoretically near-optimal, and practically state-of-the-art algorithm for the top-k subset selection {}problem.  We emphasize that the Simulator framework is concerned with how an algorithm samples, rather than its final objective. Thus, we believe that the techniques in this paper can be applied more broadly to a wide class of problems in the active learning community.

After defining terms and the setting of interest in Section~\ref{sec:prelim}, Section~\ref{sec:lower_bounds} reviews the state-of-the-art lower bounds and their limitations, and then presents our novel lower bounds for the special case when the means are known up to a permutation.
Section~\ref{MultiHypContrib} explores conditions under which $\log$ factors appear in lower bounds and we leverage these observations to prove instance-specific lower bounds for top-k subset selection in Section~\ref{sec:instance_lower_topk}.
Inspired by the lower bounds, we introduce LUCB++, the first practical, minimax-optimal algorithm for top-k subset selection in Section~\ref{sec:lucb_alg} (proofs of sample complexity guarantees are deferred to Appendix~\ref{Sec:UpperBoundProof}).
The Simulator framework and its application to the lower bound when the means are known up to a permutation is presented in Sections~\ref{sec:simulator} and \ref{Sec:Permutations}, with some proofs being deferred to Appendix~\ref{App:Permutations}. The lower bounds for top-k subset slection require more careful analysis, and are deferred to the Appendices~\ref{sec:proof_of_best_arm_subset},~\ref{sec:distinct_measure_lb}, and~ \ref{Sec:AlgRestrictions}. 
Finally, we make concluding remarks in Section~\ref{sec:conclusion}.

%% file: prelim_colt.tex
\section{Preliminaries} \label{sec:prelim}
As alluded to in the introduction, the adaptive estimation problems in this paper can be formalized as multi-arm bandits problems, where the instances $\nu  = (\nu_1,\dots,\nu_n)$ lie in an appropriate constraint set $\calS$, called an instance class (e.g., the mean vectors $(\mu_1,\dots,\mu_n)$, where $\mu_a := \Exp_{X_a \sim \nu_a}[X_a]$ lie in some specified polytope). We use the term \emph{arms} to refer both to the indices $a \in [n]$ and distributions $\nu_a$ they index. The stochastic multi-arm bandit formulation has been studied extensively in the pure-exploration setting considered in this work \citep{bechhofer1958sequential,even2006action, kalyanakrishnan2012pac,karnin2013almost,jamieson2014lil,chen2015optimal,garivier2016optimal,russo2016simple}. 
At each time $t = 1,2,\dots$, a learner plays an action $a_t \in [n]$, and observes an observation $X_{a_t,t} \in \R$ drawn i.i.d from $\nu_{a_t}$. At some time $T$, the learner decides to end the game and return some output.  Formally, let $\mathcal{F}_t$ denote the sigma-algebra generated by $\{X_{a_s,s}\}_{1 \le s \le t}$, and some additional randomness $\xi_{\Alg}$ independent of all the samples (this represents randomization internal to the algorithm). A \emph{sequential sampling algorithm} consists of
\begin{enumerate}
\item A sampling rule $(a_t)_{t \in \mathbb{N}}$, where $a_t \in [n]$ is $\mathcal{F}_{t-1}$ measurable.
\item A stopping time $T$, which is $\{\mathcal{F}_t\}_{t \in \mathbb{N}}$-measurable.
\item An output rule $\widehat{S} \subset [n]$, which is $\calF_{T}$-measurable. 
\end{enumerate}
We let $N_a(t) = \sum_{s=1}^t \I(a_s = a)$ denote the samples collected from arm $a \in \calA$ by time $t$. In particular, $N_a(T)$ is the number of times arm $a$ is pulled by the algorithm before terminating, and $\sum_{a =1}^n N_a(T) = T$. A $\MAB$ algorithm corresponds to the case where the decision rule is a singleton $\widehat{S} \in \binom{[n]}{1}$, and, more generally, a $\TopK$ algorithm specifies a $\widehat{S} \in \binom{[n]}{k}$.   We will use $\Alg$ as a variable which describes a particular algorithm, and use the notation $\Pr_{\nu,\Alg}[\cdot]$ and $\Exp_{\nu,\Alg}[\cdot]$ to denote probabilities and expectations which are taken with respect to the samples drawn from $\nu$, and the (possibly randomized) sampling, stopping, and output decisions made by $\Alg$. Finally, we adopt the following notion of correctness, which corresponds to the ``fixed-confidence'' setting in the active learning literature: 
\begin{definition} We say that a $\MAB$ algorithm is $\delta$-correct for a best-arm mapping $a^*: \calS \to [n]$ (resp $\delta$-correct for a $\TopK$ mapping $S^*: \calS \to \binom{[n]}{k}$) over an instance class $\calS$ if for all $\nu \in \calS$,  $\Pr_{\nu,\Alg}[\widehat{S} = a^*(\nu)] \ge 1- \delta$ (resp. $\Pr_{\nu,\Alg}[\widehat{S} = S^*(\nu)] \ge 1- \delta$).
\end{definition}

Typically, the best arm mapping is defined as the arm with the highest mean $a^* = \arg\max_{a \in [n]} \mu_a$, and top $k$ mapping as the arms with the $k$-largest means $\arg\max_{S \in \binom{[n]}{k}} \sum_{a \in S} \mu_a$, which captures the notion of the arm/set of arms that yield the highest reward. When the best-arm mapping returns the highest-mean arm, and the observations $X_b$ are sub-Gaussian\footnote{Formally, $X_b$ is $\sigma^2$-sub-Gaussian if $\Exp_{X_b \sim \nu_b}[e^{\lambda (X_b - \mu_b)}] \le \exp(\lambda^2\sigma^2/2)$  }, the problem complexity for $\MAB$ is typically parameterized in terms of the ``gaps'' between the means $\Delta_{b} := \mu_{a^*} - \mu_b$ \citep{mannor2004sample}. More generally, sample complexity is parametrized in terms of the $\KL(\nu_b,\nu_{a^*})$, the $\KL$ divergences between the measures $\nu_{a^*}$ and $\nu_b$. For ease of exposition, we will present our high-level contributions in terms of gaps, but the body of the work will also present more general results in terms of $\KL$'s. Finally, our theorem statements will use $\gtrsim$ and $\lesssim$ to denote inequalities up to constant factors. In the text, we shall occasionally use $\gtrsim,\lesssim,\approx$ more informally, hiding doubly-logarithmic factors in problem parameters.




%% file: oracle_colt.tex

Typically, lower bounds in the bandit and adaptive sampling literature are obtained by the change of measure technique \citep{mannor2004sample,castro2014adaptive,garivier2016optimal}.  To contextualize our findings, we begin by stating the state-of-the-art change-measure-lower bounds, as it appears in~\cite{garivier2016explore}. For a class of instances $\calS$, let $\Alt(\nu)$ denote the set of instances $\widetilde{\nu} \in \calS$ such that, $a^*(\widetilde{\nu}) \ne a^*(\nu)$. Then: 


\begin{proposition}[Theorem 1~\citep{garivier2016optimal}]\label{KaufmanDeltaMAB} If $\Alg$ is $\delta$ correct for all $\nu \in \calS$, then the expected number of samples $\Alg$ collects under $\nu$,  $\Exp_{\nu,\Alg}[T]$, is bounded below by the solution to the following optimization problem
\begin{eqnarray}\label{OracleBound1}
\min_{\tau \in \R_{\ge 0}^{n}} \sum_{a =1}^n\tau_a & \mathrm{subject}\text{ }\mathrm{to} & \inf_{\tilde{\nu} \in \Alt(\nu)}\sum_{a =1}^n \tau_a \KL(\nu_a,\tilde{\nu}_a) \ge \kl(\delta,1-\delta)
\end{eqnarray}
where $\kl(\delta,1-\delta) := \delta \log(\frac{\delta}{1-\delta}) +  (1-\delta) \log(\frac{1-\delta}{\delta})$, which scales like $\log(1/\delta)$ as $\delta \to 0$.
\end{proposition}

The above proposition says that the expected sample complexity $\Exp_{\nu,\Alg}[T]$ is lower bounded by the following, non-adaptive experiment design problem: minimize the total number of samples $\sum_a \tau_a$ subject to the constraint that these samples can distinguish between a null hypothesis $H_0 = \nu$, and any alternative hypothesis $H_1 = \tilde{\nu}$ for $\tilde{\nu} \in \nu$, with Type-I and Type-II errors at most $\delta$. We will call the optimization problem in Equation~\ref{OracleBound1} the \emph{Oracle Lower Bound}, because it captures the best sampling complexity that could be attained by a powerful ``oracle'' who knows how to optimally sample under $\nu$. 


Unlike the oracle, a real learner would never have access to the true instance $\nu$. Indeed, for $\MAB$ instances with sufficient structure, Equation~\ref{OracleBound1} gives a misleading view of the instrinsic difficulty of the problem. For example, let $\calS$ denote the class of instances $\nu$ where $\nu_a = \calN(\mu_a,1)$, and $\mu$ lies in the simplex, i.e. $\mu_a \ge 0$ and $\sum_{a \in \calA} \mu_a = 1$. If the ground truth instance $\nu^*$ has $\mu_{a^*} = .9$ for some $a^* \in [n]$, then any oracle which uses the knowledge of the ground truth to construct a sampling allocation can simply put all of its samples on arm $a^*$. Indeed, the simplex constraint implies that $a^*$ is indeed the best arm of $\nu$, and that any instance $\widetilde{\nu}$ which has a best arm other than $a^*$ must have $\widetilde{\nu}_{a^*} < .5$. Thus, $\forall~\widetilde{\nu} \in \Alt(\nu)$, $\KL(\nu^*_{a^*},\widetilde{\nu}^*) \ge \frac{(.9-.5)^2}{2} = \Omega(1)$. In other words, the sampling vector
\begin{eqnarray}\tau_a =\begin{cases} (.08)^{-1}\kl(\delta,1-\delta)  & a = a^* \\ 0 & a \ne a^* \end{cases} 
\end{eqnarray}
is feasible for Equation~\ref{OracleBound1} which means that the optimal number of samples predicted by Equation~\ref{OracleBound1} is no more than $\sum_a \tau_a = \tau_{a^*} = O(\log(1/\delta))$. But this predicted sample complexity doesn't depend on the number of arms! 

So how how hard is the simplex really? To address this question, we prove the first lower bound in the literature which, to the author's knowledge~\footnote{Before publication this work, but after a preprint was available online, this result was obtained independently by~\cite{1702.03605}}, accurately characterizes the complexity a strictly easier problem: when the means are known up to a permutation. Because the theorem holds when the measures are known up to a permutation, it also holds in the more general setting when the measures satisfy any permutation-invariant constraints, including when \textbf{a)} the means lie on the simplex \textbf{b)} the means lie in an $l_p$ ball or \textbf{c)} the vector $\mu_{(1)} \ge \mu_{(2)} \ge \dots \mu_{(n)}$ of sorted means satisfy arbitrary constraints (e.g. weighted $l_p$ constraints on the sorted means \citep{bogdan2013statistical}).

In what follows, let $\bS_n$ denote the group of permutations on $[n]$ elements and $\pi(j)$ denote the index which $j$ is mapped to under $\pi$. For an instance $\nu = (\nu_1,\dots,\nu_n)$, we let $\pi(\nu) = \{\nu_{\pi(1)},\dots, \nu_{\pi(n)}\}$, and define the instance class $\bS_n(\nu) := \{\pi(\nu), \pi \in \bS_n\}$. Moreover, we use the notation $\pi \sim \bS_n$ to denote that $\pi$ is drawn uniformly at random. 
With this notation, $N_{\pi(b)}$ is the number of times we pull the arm indexed by $\pi(b) \in [n]$, i.e. the samples from $\nu_{\pi(b)}$. And $\Exp_{\pi \sim \bS_n}[N_{\pi(b)}(T)]$ is the expected number of samples from $\nu_b$ since $(\pi(\nu))_{\pi(b)} $ is always equal to $\nu_b$, and \emph{not} the distribution $\nu_{\pi(b)}$.
The following theorem essentially says that if the instance is randomly permuted before the start of the game, no $\delta$-correct algorithm can avoid taking a substantial number of samples from $\nu_b$ for any $b \in [n]$.
\begin{theorem}[Lower bounds on Permutations]\label{Thm:PermLB}
	Let $\nu$ be an $\MAB$ instance with unique best arm $a^*$, and for $b \ne a^*$, define $\tau_b = \frac{1}{\KL(\nu_{a^*},\nu_{b}) +\KL(\nu_{b},\nu_{a^*})}$. If $\Alg$ is $\delta$-correct over $\bS_n(\nu)$ then
	\begin{eqnarray}\label{Eq:PermutationIndividual}
	\E_{\pi \sim \bS_n}\Pr_{\pi(\nu),\Alg}[N_{\pi(b)}(T) > \tau_{b} \log(1/4\eta)] \ge \eta - \delta
	\end{eqnarray}
	for any $\nu \in (\delta,1/4)$, and by Markov's inequality
	\begin{eqnarray}\label{Eq:PermutationTotal}
	\Exp_{\pi \sim \bS_n}\Exp_{\pi(\nu),\Alg}[T]= \Exp_{\pi \sim \bS_n} \left[ \sum_{b \neq a^*} \Exp_{\pi(\nu),\Alg}[ N_{\pi(b)}(T)] \right] \geq \sup_{\eta \in [\delta, 1/4]} (\eta - \delta) \log(1/4\eta)  \sum_{b \neq a^*}\tau_b.  
	\end{eqnarray}
	In particular, if $\Alg$ is $\delta \le 1/8$-correct, then $\displaystyle\Exp_{\pi \sim \bS_n}\Exp_{\pi(\nu),\Alg}[T] \gtrsim  \sum_{b \neq a^*}\frac{1}{\KL(\nu_{a^*},\nu_{b}) +\KL(\nu_{b},\nu_{a^*})}$.
\end{theorem}
	The proof of the above result is found in Section~\ref{sec:proof_of_perm_lb} and follows from the application of the Simulator introduced in Section~\ref{sec:simulator}, and machinery developed throughout Section~\ref{Sec:Permutations}. When the reward distributions are $\nu_b = \mathcal{N}(\mu_b,1)$, $\KL(\nu_{a^*},\nu_b) = \KL(\nu_b,\nu_{a^*}) = \frac{1}{2}\Delta_b^2$ (recall $\Delta_b = \mu_{a^*} - \mu_b)$.  Moreover, applying the oracle bound of Proposition~\ref{KaufmanDeltaMAB} to permutations implies a lower bound of  $\gtrsim \max_{b \ne a^*} \Delta_b^{-2} \log(1/\delta)$. Indeed, for each $b \in [n] \setminus \{a^*\}$, one would need to take enough samples to distinguish $\nu$ from the alternative instance where the means $\mu_{a^*}$ and $\mu_b$ are swapped, with probability of error at most $1-\delta$. Hence, combining this oracle lower bound with Theorem~\ref{Thm:PermLB}  yields 
	\begin{eqnarray}\label{total_sample_complexity}\Exp_{\pi \sim \bS_n}\Exp_{\pi(\nu),\Alg}[T] \gtrsim \max\{ \max_{b \ne a^*} \Delta_b^{-2} \log(1/\delta) , \sum_{b \neq a^*}\Delta_b^{-2} \} ~.\end{eqnarray}
	For comparison, the bound of Proposition~\ref{KaufmanDeltaMAB} only implies a lower bound of $\gtrsim \max_{b \ne a^*} \Delta_b^{-2} \log(1/\delta)$, since an oracle who knows how to sample could place all their samples on $a^*$. Thus, for constant $\log(1/\delta)$, our lower bound differs from the bound in Proposition~\ref{KaufmanDeltaMAB} by up to a factor of $n$, the number of arms. In particular, when the gaps are all on the same order, the $\delta \to 0$ asymptotics only paint an accurate picture of the sample complexity once $\delta$ is exponentially-small in $n$. 

	In fact, our lower bound is essentially unimproveable: Appendix~\ref{sec:permu_upper_bounds} provides an upper bound for the setting where the top-two means are known, whose expected sample complexity on any permutation matches the on-average complexity in Equation~\ref{total_sample_complexity} up to constant and doubly-logarithmic factors. Together, these upper and lower bounds depict two very different regimes:
	\begin{enumerate}
		\item Treating $\delta$ as a fixed constant, the lower bound of the constrained problem essentially matches known upper bounds for the \emph{unconstrained} best-arm problem \citep{chen2015optimal,jamieson2014lil}. Thus, in this regime, {\em knowing the instance up to a permutation of the arms does not affect the sample complexity}. 
		\item As $\delta \to 0$, an algorithm which knows the means up to a permutation can learn to optimistically and aggressively focus its samples on the top arm, yielding an asymptotic sample complexity predicted by Proposition~\ref{KaufmanDeltaMAB}, one which is potentially far smaller than that of the unconstrained problem.\footnote{In fact, using a track-and-stop strategy similar to~\cite{garivier2016optimal} one could design an algorithm which matches the constant factor in Proposition~\ref{KaufmanDeltaMAB}.}
	\end{enumerate}
	These two regimes show that the Simulator and oracle lower bounds are \emph{complementary}, and go after two different aspects of problem difficulty:
	In the second regime, the oracle lower bound characterizes $\lesssim \max_{b \ne a^*} \Delta_b^{-2} \log(1/\delta)$ samples sufficient to \emph{verify} that arm $a^*$ is the best, whereas in the first regime, the Simulator characterizes the $\gtrsim \sum_{b \neq a^*}\Delta_b^{-2}$ samples needed to learn a favorable sampling allocation\footnote{The simulator also provides a lower bound on the \emph{tail} of the number of pulls from a suboptimal arm since, with probability $\delta$, arm $b$ is pulled $\tau \log(1/8\delta)$ times. This shows that even though you can learn an oracle allocation on average, there is always a small risk of oversampling. Such affects do not appear from Proposition~\ref{KaufmanDeltaMAB}, which only control the number of samples taken in expectation\label{tail_footnote}}. We remark that Garivier et al. \cite{garivier2016explore} also explores the problem of learning-to-sample by establishing the implications of Proposition~\ref{KaufmanDeltaMAB} for finite-time regret; however, there approach does not capture any effects which aren't reflected in Proposition~\ref{KaufmanDeltaMAB}. Moreover, Bubeck and Cesa-Bianchi~\citep{bubeck2012regret} establish an \emph{minimax}, rather than instance-specific, lower bound for regret by considering permutations of a simple $\MAB$ instance where $n-1$ arms have the name mean, and one arm has a slightly elevated mean. Finally, we note that proving a lower bound for learning a favorable strategy in our setting must consider some sort of average or worst-case over the instances. Indeed, one could imagine an algorithm that starts off by pulling the first arm $1$ until it has collected enough samples to test whether $\mu_1 = \mu_{a^*}$ (i.e. $\mu > \max_{b \ne a^*} \mu_b$ ), and then pulling arm $2$ to test whether $\mu_2 = \mu_{a^*}$, and so on. If arm $1$ is the best, this algorithm can successfully identify it without pulling any of the others, thereby matching the oracle lower bound.

%% file: multi_hyp_contributions_colt.tex
\subsection{Sharper Multiple-Hypothesis Lower Bounds~\label{MultiHypContrib}}

In contrast to change-of-measure type lower bounds like Proposition~\ref{KaufmanDeltaMAB}, the active PAC learning literature (e.g., binary classification) leverages classical tools like Fano's inequality with packing arguments \citep{castro2008minimax,raginsky2011lower} and  other measures of class complexity such as the disagreement coefficient \citep{hanneke2009theoretical}. Because these arguments consider multiple hypotheses simultaneously, they can capture effects which the worst-case binary-hypothesis oracle lower bounds like Equation~\ref{OracleBound1} can miss. 

While the considerable gap between two-way and multiple tests is well-known in the passive setting \citep{tsybakov2009introduction}, existing techniques which capture this multiple-hypothesis complexity lead to coarse, worst- or average-case lower bounds for adaptive problems because they rely on constructions which are either artificially symmetric, or are highly pessimistic  \citep{castro2008minimax,raginsky2011lower,kalyanakrishnan2012pac}. Moreover, the constructions rarely shed insights on \emph{why} active learning algorithms seem to avoid paying the costs for  multiple hypotheses that would occur in the passive setting, e.g. the folk theorem: ``active learning removes log factors'' \citep{castro2014adaptive}.

	As a first step towards understanding these effects, we prove the first instance-based lower bound which sheds light on why active learning is able to effectively reduce the number of hypotheses it needs to distinguish. To start, we prove a qualitative result for a simplified problem, using a novel reduction to Fano's inequality via the simulator. The following theorem is proved in Appendix~\ref{sec:proof_of_best_arm_subset}:
	\begin{theorem}\label{Thm:BestArmSubset}
	Let $\Alg$ be $1/8$-correct, consider a game with best arm $\nu_1$ and $n-1$ arms of measure $\nu_2$. Let $S_m := \{a \in [n]: N_a(T) > \frac{1}{16}(\KL(\nu_1,\nu_2) + \KL(\nu_2,\nu_1)) \log\frac{n}{2^{16}m} \}$. Then
	\begin{eqnarray}
	&&\label{contrib:secondline} \Pr_{\pi \sim \bS_n} \Pr_{\pi(\nu),\Alg}\left[ \{\pi(1) \in S_m\} \wedge \{|S_m| \ge m\}\right] \ge \frac{3}{4}
	\end{eqnarray}
	\end{theorem} 
	For Gaussian rewards with unit variance, $\KL(\nu_1,\nu_2) + \KL(\nu_2,\nu_1) = \Delta^{2}$, where $\Delta$ is the gap between the means $\mu_1 - \mu_2$, the above proposition states that, for any $m \in [n]$, any correct $\MAB$ algorithm must sample some $m$ arms, including the top arm, $\tau \gtrsim \Delta^{-2}\log (n/m)$ times. Thus, the number of samples allocated by the oracle of Proposition~\ref{KaufmanDeltaMAB} are necessarily insufficient to identify the best arm for moderate $\delta$. This is because, until sufficiently many samples has been taken, one cannot distinguish between the best arm, and other arm exhibiting large statistical deviations. Looking at exponential-gap style upper bounds~\citep{chen2015optimal,karnin2013almost}, which halve the number of arms in consideration at each round, we see that our lower bound is qualitatively sharp for some algorithms\footnote{We believe that UCB-style algorithms exhibit this same qualitative behavior}. Further, we emphasize that this set of $m$ arms which must be pulled $\tau$ times may be random\footnote{In fact, for an algorithm with which only samples $m' = O(m)$ arms $\tau \gtrsim \Delta^{-2}\log (n/m)$, this subset of arms \emph{must} be random. This is because for a fixed subset of $m'$ arms, one could apply Theorem~\ref{Thm:BestArmSubset} to the remaining $n - m'$ arms.}, depend on the random fluctuations in the samples collected, and thus cannot be determined using knowledge of the instance alone. Stated otherwise, if one sampled according to the \emph{proporitions} as ascribed by Proposition~\ref{KaufmanDeltaMAB}, then the total number of samples one would need to collect would be suboptimal (by a factor of $\log n$). Thus, effective adaptive sampling should adapt its allocation to the statistical deviations in the collected data, not just the ground truth instance. We stress that the Simulator is indispensable for establishing this result, because it lets us characterize the stage-wise sampling allocation of adaptive algorithms.

	Guided by this intuition, Appendix~\ref{sec:distinct_measure_lb} employs a more involved proof strategy to establish the following guarantee for $\MAB$ with Gaussian rewards (a more general result for single-parameter exponential families is given by Theorem~\ref{Big Main Theorem} in Appendix~\ref{sec:main_tec_theorem}):
	\begin{proposition}[Lower Bound for Gaussian $\MAB$] \label{GaussianProp}Suppose $\nu = (\nu_1,\dots,\nu_n)$ has measures $\nu_a = \mathcal{N}(\mu_a,1)$,with  $\mu_1 > \mu_2 \ge \dots \mu_n$. Then, if $\Alg$ is $\delta \le 1/16$ correct over $\bS_n(\nu)$,
	\begin{eqnarray}
	\Exp_{\pi \sim \bS_n}\Exp_{\pi(\nu^{(1)}),\Alg}[N_{\pi(1)}(T)] &\gtrsim& \max_{2 \le m \le n} \Delta_{m}^{-2} \log (m/\delta) \quad \text{where } \Delta_m = \mu_1 - \mu_m
	\end{eqnarray}
	\end{proposition}
	In particular, when all the gaps are on the same order $\Delta$, then the top arm must be pulled $\Omega(\Delta^{-2} \log n)$ times. When the gaps are different, $\max_{2 \le m \le n} \Delta_{m}^{-2} \log m$ trades off between larger $\log m$ factor as the inverse-gap-squared $\Delta_m^{-2}$ shrinks. 
	As we explain in Appendix~\ref{TiltingsExplanation},  this tradeoff is best understood in the sense that the algorithm is conducting an \emph{instance-dependent} union bound, where the union bound places more confidence on means closer to the top. The proof itself is quite involved, and constitutes the main technical contribution of this paper. We devote Section~\ref{TiltingsExplanation} to explaining the intuition and proof roadmap. Our argument makes use of ``tilted distributions'', which arise in Herbst Argument in Log-Sobolev Inequalities in the concentration-of-measure literature~\citep{raginsky2014concentration}. Tiltings translate the tendency of some empirical means to deviate far above their averages (i.e. to anti-concentrate) into a precise information-theoretic statement that they ``look like'' draws from the top arm. To the best of our knowledge, this constitutes the first use of tiltings to establish information-theoretic lower bounds, and we believe this strategy may have broader use. 

	
\subsection{Instance-Specific Lower bound for $\TopK$}\label{sec:instance_lower_topk}
	Proposition~\ref{GaussianProp} readily implies the first instance-specific lower bound for the $\TopK$. The idea is that, if I can identify an arm $j \in [k]$ as one of the top $k$ arms, then, in particular, I can identify arm $j$ as the best arm among $\{j\} \cup \{k+1,\dots,n\}$. Similarly, if I can reject arm $\ell$ as not part of the top $k$, then I can identify it as the ``worst'' arm among $\{1,\dots,k\} \cup \{\ell\}$. Section~\ref{Sec:AlgRestrictions} formally proves the following lower bound using by applying the above eduction to Proposition~\ref{GaussianProp}:
	\begin{proposition}[Lower Bound for Gaussian $\TopK$]\label{TopKGaussianProp} Suppose $\nu = (\nu_1,\dots,\nu_n)$ has measures $\nu_a = \mathcal{N}(\mu_a,\theta)$, with  $\mu_1 \ge \mu_2 \ge \dots \mu_k > \mu_{k+1} \ge \dots \mu_n$. Then, if $\Alg$ is $\delta \le 1/16$ correct over $\bS_n(\nu)$,
	\begin{eqnarray}\label{granular_top_k_eq}
	\Exp_{\pi \sim \bS_n}\Exp_{\pi(\nu^{(1)}),\Alg}[N_{\pi(j)}(T)] &\gtrsim& \begin{cases} \max_{m > k} (\mu_j - \mu_m)^{-2}\log((m - k + 1)/\delta) & j \le k \\
	\max_{m \le k} (\mu_j - \mu_m)^{-2}\log((k+2 - m)/\delta) & j > k
	\end{cases}
	\end{eqnarray}
	\end{proposition}

	By taking $m = k+1$ and $m = k$ in the first and second lines of~\ref{granular_top_k_eq}, our result recovers the gap-dependent bounds of \cite{kalyanakrishnan2012pac} and \cite{luedtke2016asymptotically} . Moreover, when the gaps are on the same order $\Delta$, we recover the worst-case lower bound from~\cite{kalyanakrishnan2012pac} of $k\Delta^{-2}\log(n-k) + (n-k)\Delta^{-2}\log k$. 
\subsubsection{Comparison with Chen et al.~\cite{1702.03605}}
	After a manuscript of the present work was posted on one of its authors' websites, ~\cite{1702.03605} presented an alternative proof of Proposition~\ref{TopKGaussianProp}, also by a reduction to $\MAB$. Instead of tiltings, their argument handles different gaps by a series of careful reductions to a symmetric $\MAB$ problem, to which they apply Proposition~\ref{OracleBound1}. As in this paper, their proof hinges on a ``simulation'' argument which compares the behavior of an algorithm on an instance $\nu$ to a run of an algorithm where the reward distributions change mid-game. This seems to suggest that our simulator framework is in some sense a natural tool for these sorts of lower bounds. 

	While our works prove many of the same results, our papers differ considerably in emphasis. The goal for in this work is to explain \emph{why} algorithms must incur the sample complexities that they do, rather than just sharpen logarithmic factors. In this vein, we establish Theorem~\ref{Thm:BestArmSubset}, which has no analogue in~\cite{1702.03605}. Moreover, we believe that the proof of Proposition~\ref{GaussianProp} based on tiltings is a step towards novel lower bounds for more sophisticated problems by translating intuitions about large-deviations into precise, information-theoretic statements. Further still, our Theorem~\ref{Thm:PermLB} (and Proposition~\ref{Appendix:BestArmProp} in the appendix) imply lower bounds on the tail-deviations of the number of times suboptimal arms need to be sampled in constrained problems (see footnote~\ref{tail_footnote}). 

%% file: upper_bounds.tex
\section{LUCB++} \label{sec:lucb_alg}

The previous section showed that for $\TopK$ in the worst case, the bottom $(n-k)$ arms must be pulled in proportion to $\log(k)$ times while the top $k$ arms must be pulled in proportion to $\log(n-k)$ times. Inspired by these new insights, the original LUCB algorithm of \cite{kalyanakrishnan2012pac}, and the analysis of \cite{jamieson2014lil} for the $\MAB$ setting, in this section we propose a novel algorithm for $\TopK$: LUCB++. The LUCB++ algorithm proceeds exactly like that of \cite{kalyanakrishnan2012pac}, the only difference being the definition of the confidence bounds used in the algorithm. 

At each round $t = 1,2,\dots$, let $\widehat{\mu}_{a,N_a(t)}$ denote the empirical mean of all the samples from arm $a$ collected so far. Let $U(t,\delta) \propto \sqrt{\tfrac{1}{t}\log(\log(t)/\delta)}$ be an anytime confidence bound based on the law of the iterated logarithm (see \citet[Theorem 8]{kaufmann2014complexity} for explicit constants). Finally, we let $\TOP_t$ denote the set of the $k$ arms with the largest empirical means. The algorithm is outlined in Figure~\ref{lucb++}, and satisfies the following guarantee:

\begin{algorithm}[t]
	 	\textbf{Input} Set size $k$, confidence $\delta$, confidence interval $U(\cdot,\delta)$\\
	 	\textbf{Play} Each arm $a \in [n]$ once \\
 		\textbf{For} rounds $t = n+1,n+2,\dots$ \\
 		\Indp \textbf{Let} $\mathrm{TOP}_t = \arg\max_{ S \subset [n]: |S|=k } \sum_{i \in S} \widehat{\mu}_{a,N_a(t)}$, \\
 		\textbf{If} the following holds, \textbf{Then} return $\mathrm{TOP}_t$:
 		\begin{align}\label{eqn:stopping_time}
		\min_{a \in \mathrm{TOP}_t} \widehat{\mu}_{a,N_a(t)} - U(N_a(t),\tfrac{\delta}{2(n-k)}) > \max_{a \in [n]-\mathrm{TOP}_t} \widehat{\mu}_{a,N_a(t)} + U(N_a(t),\tfrac{\delta}{2k})
		\end{align}
		\\
 		\textbf{Else} pull $h_t$ and $l_t$, given by:
 		\begin{align*}
		h_t := \min_{a \in \mathrm{TOP}_t} \widehat{\mu}_{a,N_a(t)} - U(N_a(t),\tfrac{\delta}{2(n-k)}) \qquad l_t := \max_{a \in [n]-\mathrm{TOP}_t} \widehat{\mu}_{a,N_a(t)} + U(N_a(t),\tfrac{\delta}{2k}).
		\end{align*}
		\caption{LUCB++\label{lucb++}}
	\end{algorithm}

\begin{theorem}\label{Thm:LUCBplus}
Suppose that $X_a \sim \nu_a$ is $1-$subgaussian. Then, for any $\delta \in (0,1)$, the LUCB++ algorithm is $\delta$-correct, and the stopping time $T$ satisfies 
\begin{align*}
T \leq \sum_{i=1}^k c \Delta_i^{-2} \log( \tfrac{ (n-k) \log(\Delta_i^{-2})}{\delta} ) + \sum_{j=k+1}^n c \Delta_j^{-2} \log( \tfrac{ k \log(\Delta_j^{-2})}{\delta} )
\end{align*}
with probability at least $1-\delta$, where $c$ is a universal constant.
\end{theorem}

By Propositions~\ref{TopKGaussianProp} we recognize that when the gaps are all the same the sample complexity of the LUCB++ algorithm is unimprovable up to $\log\log(\Delta_i)$ factors.
This is the first {\em practical} algorithm that removes extraneous log factors on the sub-optimal $(n-k)$ arms \cite{kalyanakrishnan2012pac,chen2016pure}.
However, it is known that not all instances must incur a multiplicative $\log(n-k)$ on the top $k$ arms \cite{chen2016pure,1702.03605}.
Indeed, when $k=1$ this problem is just the best-arm identification problem and the sample complexity of the above theorem, ignoring doubly logarithimc factors, scales like $\log(n/\delta) \Delta_1^{-2} + \log(1/\delta) \sum_{i=2}^n \Delta_i^{-2}$. But there exist algorithms for this particular best-arm setting whose sample complexity is just $\log(1/\delta) \sum_{i=1}^n \Delta_i^{-2}$ exposing a case where Theorem~\ref{Thm:LUCBplus} is loose \cite{karnin2013almost,jamieson2014lil,chen2015optimal,chen2016pure}.
In general, this additional $\log(n-k)$ factor is unnecessary on the top $k$ arms when $\sum_{i=1}^k \Delta_i^{-2} \gg \sum_{i=k+1}^n \Delta_i^{-2}$, but for large $n$, this is a case unlikely to be encountered in practice.

While this manuscript was in preparation, \cite{1702.03605} proposed a $\TopK$ algorithm which satisfies stronger theoretical guarantees, essentially matching the lower bound in Theorem~\ref{TopKGaussianProp}. However, their algorithm (and the matroid-bandit algorithm of~\cite{chen2016pure}) relies on exponential-gap elimination, making it unsuitable for practical use\footnote{While exponential-gap elimination algorithms might have the correct dependence on problem parameters, their constant-factors in the sample complexity are incredibly high, because they rely on the median-elimination as a subroutine (see~\cite{jamieson2014lil} for discussion)}. Furthermore, our improved LUCB++ confidence intervals can be reformulated for different KL-divergences, leading to tighter bounds for non-Gaussian rewards such as Bernoullis. Moreover, we can ``plug in'' our LUCB++ confidence intervals into other LUCB-style algorithms, sharpening their $\log$ factors. For example, one could ammend the confidence intervals in the CLUCB algorithm of~\cite{chen2014combinatorial} for combinatorial bandits, which would yield slight improvements for arbitrary decision classes, and near-optimal bounds for matroid classes considered in~\citep{chen2016pure}.

\begin{table}
\centering
\begin{tabular}{|c|c|c|c|c|}
\hline
$n$ &LUCB++&LUCB  &Oracle &Uniform\\ \hline
$10^1$ &1.0&0.99&1.60&1.67\\ \hline
$10^2$ &1.0 &1.17&2.00&3.4\\ \hline
$10^3$ &1.0&1.50&2.51&5.32\\ \hline
$10^4$ &1.0&1.89&2.90&7.12\\ \hline
$10^5$ &1.0&2.09&3.32&8.49\\ \hline
\end{tabular}
\caption{The number of samples taken by the algorithms before reaching their stopping condition, relative to LUCB++.}
\label{tab:sim1}
\end{table}

To demonstrate the effectiveness of our new algorithm we compare to a number of natural baselines: LUCB of \cite{kalyanakrishnan2012pac}, a $\TopK$ version of the oracle strategy of \cite{garivier2016optimal}, and uniform sampling; all three use the stopping condition of \cite{kalyanakrishnan2012pac} which is when the empirical top $k$ confidence bounds\footnote{To avoid any effects due to the particular form of the any-time confidence bound used, we use the same finite-time law-of-the-iterated logarithm confidence bound used in \cite[Theorem 8]{kaufmann2014complexity} for all of the algorithms.} do not overlap with the bottom $n-k$, employing a union bound over all $n$ arms.
Consider a $\TopK$ instance for $k=5$ constructed with unit-variance Gaussian arms with $\mu_i = 0.75$ for $i \leq k$ and $\mu_i = 0.25$ otherwise. 
Table~\ref{tab:sim1} presents the average number of samples taken by the algorithms before reaching the stopping criterion, relative to the the number of samples taken by LUCB++.
For these means, the oracle strategy pulls each arm $i$ a number of times proportional to $w_i$ where $w_i = \frac{ \sqrt{n/k-1} -1}{n-2k}$ for $i \leq k$ and $w_i = \frac{1-k w_k}{n-k}$ for $i > k$ ($w_i = 1/n$ for all $i$ when $n=2k$).
Note that the uniform strategy is indentical to the oracle strategy, but with $w_i = 1/n$ for all $i$.

%% file: simulator.tex
\section{Lower Bounds via The Simulator~\label{sec:simulator}}




As alluded to in the introduction, our lower bounds treat adaptive sampling decisions made by the algorithm as hypothesis tests between different instances $\nu$. Using a type of gadget we call a \emph{Simulator}, we reduce lower bounds on \emph{adaptive} sampling strategies to a family of lower bounds on different, possibly data-dependent and time-specific \emph{non-adaptive} hypothesis testing problems.

The Simulator acts as an adversarial channel intermediating between the algorithm $\Alg$, and i.i.d samples from the true instance $\nu$. Given an instance $\nu$, let $\Trs = \{X_{[a,s]}\}_{a \in [n], s \in \N} \in \R^{n\times \mathbb{Z}_{\ge 0}}$ denote a random transcript of an infinite sequence of samples drawn i.i.d from $\nu$, where $\Trs_{a,s} = X_{[a,s]} \overset{iid}{\sim} \nu_a$. We can think of any sequential sampling algorithm $\Alg$ as operating by interacting with the transcript, where the sample $X_{a_t,t}$ is obtained by reading the sample $X_{[a_t,N_{a_t}(t)]}$ off from $\Trs$ (recall that $N_{a}(t)$ is the number of times arm $a$ has been pulled at the end of round $t$). With this notation, we define a simulator as follows:
\begin{definition}[Simulator]
A simulator $\Sim$ is a map which sends $\Trs$ to a modified transcript $\Trshat = \{\Xhat_{[a,s]}\}_{a \in [n], s \in \N}$, which $\Alg$ will interact with instead of $\Trs$ (Figure 1). We allow this mapping to depend on the ground truth $\nu$ and some internal randomness $\xi_{\Sim}$. 
\end{definition}


\begin{figure}\label{SimulatorDiagram}
  \centering
    \includegraphics[width=0.65\textwidth]{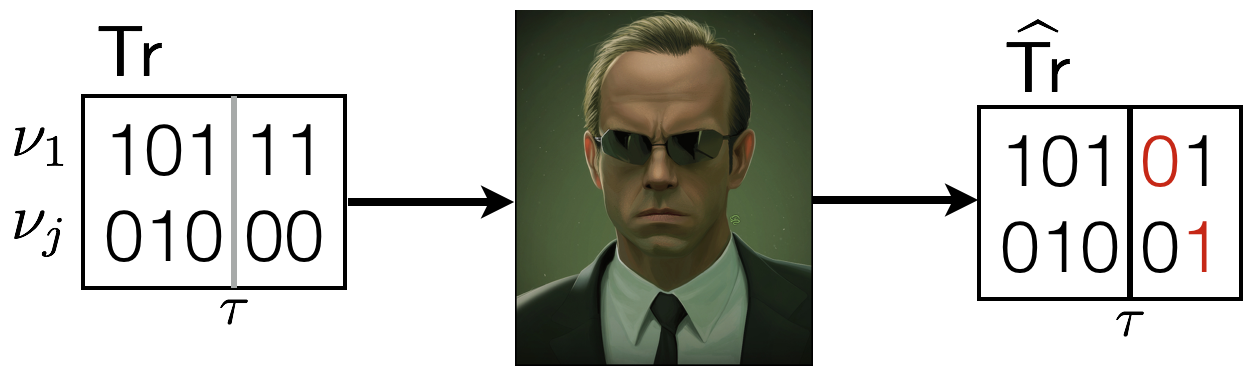}
    \caption{The Simulator acts as a man-in-the-middle between the original transcript and the transcript the algorithm receives. It leaves the transcript unchanged before some time $\tau$, but modifies it in arbitrary ways after this time. Red denotes the samples that were changed that reduced the distance between the instances. Note that all events defined on just the first $\tau$ samples are truthful.}
\end{figure}


Equivalently, $\Sim(\nu)$ is a measure on a \emph{random process} $\Trshat = \{\widehat{X}_{[a,s]}\}_{a \in [n], s \ge 1}$, which, unlike $\nu$, does not require the samples $\widehat{X}_{[a,1]},\widehat{X}_{[a,2]},\dots$ to be i.i.d (or even independent). Hence, we use the shorthand $\Sim(\nu)$ to refer the measure corresponding to $\Pr_{\Sim(\nu),\Alg}$, and let $\Pr_{\Sim(\nu),\Alg}$ denote the probability taken with respect to $\Sim$'s modified transcript $\Trshat$, and the internal randomness in $\Alg$ and $\Sim$. With this notation, the quantities $\TV(\Sim(\nu),\Sim(\nu'))$ and $\KL(\Sim(\nu),\Sim(\nu'))$ are well defined as the $\TV$ and $\KL$ divergences of the random process $\Trshat$ under the measures $\Sim(\nu)$ and $\Sim(\nu')$.


Note that, in general, $\TV(\nu,\nu') = \KL(\nu,\nu') = \infty$ if $\nu_a \ne \nu_a'$ for some $a$, since $\nu$ (resp $\nu'$) govern an infinite i.i.d sequence $\{X_{[s,a]}\}\sim \nu_a$ (resp $\sim \nu_a'$). However, in this paper we will always design our simulator so that the quantity $\KL(\Sim(\nu),\Sim(\nu'))$ is finite, and in fact quite small. The hope is that if the modified transcript $\Trshat$ conveys too little information to distinguish between $\Sim(\nu^{(1)})$ and $\Sim(\nu^{(2)})$, then $\Alg$ will have to behave similarly on both simulated instances. Hence, we will show that if $\Alg$ behaves differently on two instances $\nu^{(1)}$ and $\nu^{(2)}$, yet $\Sim$ limits information $\KL$ between them, then $\Alg$'s behavior must differ quite a bit under $\nu^{(i)}$ versus $\Sim(\nu^{(i)})$, for either $i = 1$ or $i=2$. Formally, we will show that $\Alg$ will have to ``break'' the simulator, in the following sense:
	\begin{definition}[Breaking] Given measure $\nu$, algorithm $\Alg$, and simulator $\Sim$, we say that $W \in \mathcal{F}_T$ is a truthful event under $\Sim(\nu)$  if, for all events $E \in \mathcal{F}_T$, 
	\begin{eqnarray}
	\Pr_{\Sim(\nu),\Alg}[E \wedge W] = \Pr_{\nu,\Alg}[E \wedge W]
	\end{eqnarray}
	On the other hand, we will say that $\Alg$ breaks on $W^c$ under $\Sim(\nu)$. Recall that $\mathcal{F}_t$ is the $\sigma$-algebra generated by $\xi_{\Alg}$, and the actions/samples collected by $\Alg$ up to time $t$.
	\end{definition}
	The key insight is that, whenever $\Sim(\nu)$ doesn't break (i.e. on a truthful event $W$), a run of $\Alg$ on $\nu$ can be perfectly simulated by running $\Alg$ on $\Sim(\nu)$. But if $\Sim(\nu)$ fudges $\Trs$ in a way that drastically limits information about $\nu$, this means that $\Alg$ can be simulated using little information about $\nu$, which will contradict information theoretic lower bounds. This suggests the following recipe for proving lower bounds:

	\textbf{1)} State a claim you wish to falsify over a class of instances $\nu \in \mathcal{S}$ (e.g., the best arm is not pulled more than $\tau$ times, with some probability  ). \textbf{2)} Phrase your claims as candidate truthful events on each instance (e.g. $W_{\nu} := \{N_{a^*(\nu)}(T) \le \tau\}$ where $a^*(\nu)$ is the best arm of $\nu$)
	\textbf{3)} Construct a simulator $\Sim$ such that $W_{\nu}$ is truthful on $\Sim(\nu)$, but $\KL_{\Alg}(\Sim(\nu),\Sim(\widetilde{\nu}))$ (or $\TV$) is small for alternative pairs $\nu,\widetilde{\nu}$. For example, if the truthful event is $\{N_{a^*(\nu)}(T) \le \tau\}$, then simulator should only modify samples $X_{[a^*,\tau+1]},X_{[a^*,\tau+2]},\dots$. \textbf{4)} Apply an information-theoretic lower bound (e.g.,  Proposition~\ref{Prop:SimLeCam} to come) to show that the simulator breaks (e.g. $\Pr_{\nu,\Alg}[W_\nu^c]$ is large for at least one $\nu \in \calS$, or for a $\nu$ drawn uniformly from $\calS$).


%% file: two_hyp.tex

\section{Applying the Simulator to Permutations\label{Sec:Permutations}}
In what follows, we show how to use the simulator to prove Theorem~\ref{Thm:PermLB}. At a high level, our lower bound follows from considering \emph{pairs} of instances where the best arm is swapped-out for a sub-optimal arm, and ultimately averaging over those pairs. On each such pair, we apply a version of Le Cam's method to the simulator setup (proof in Section~\ref{sec:simlecamproof}):
	\begin{proposition}[Simulator Le Cam]\label{Prop:SimLeCam}
	Let $\nu^{(1)}$ and $\nu^{(2)}$ be two measures, $\Sim$ be a simulator, and let $W_i$ be two truthful events under $\Sim(\nu^{(i)})$ for $i = 1,2$. Then, for any algorithm $\Alg$
	\begin{eqnarray}\label{Eq:LeCam}
	\sum_{i=1}^2 \Pr_{\nu^{(i)},\Alg}(W_i^c) \ge \sup_{E \in \calF_T} |\Pr_{\nu^{(1)},\Alg}(E) - \Pr_{\nu^{(2)},\Alg}(E)|  - Q\left(\KL_{\Alg}\left(\Sim(\nu^{(1)}),\Sim(\nu^{(2)})\right)\right)~,
	\end{eqnarray}
	where $Q(\beta) = \min\left\{1 - \frac{1}{2}e^{-\beta},  \sqrt{\beta/2}\right\}$. The bound also holds with $Q\left(\KL_{\Alg}\left(\Sim(\nu^{(1)}),\Sim(\nu^{(2)})\right)\right)$ replaced by $\TV_{\Alg}\left(\Sim(\nu^{(1)}),\Sim(\nu^{(2)})\right)$.
	\end{proposition}
	Note that Equation~\ref{Eq:LeCam} decouples the behavior of the algorithm under $\nu$ from the information limited by the simulator. This proposition makes formal the intuition from Section~\ref{sec:simulator} that the algorithm which behaves differently on two distinct instances must ``break'' any simulator that severely limits the information between them.


\subsection{Lower Bounds on 1-Arm Swaps }
	The key step in proving Theorem~\ref{Thm:PermLB} is to establish a simple lower bound that holds for pairs of instances obtained by ``swapping'' the best arm. 
	\begin{proposition}\label{cor:2arm}
	Let $\nu$ be an instance with unique best arm $a^*$. For $b \in [n]-\{a^*\}$, let  $\nu^{(b,a^*)}$  be the instance obtained by swapping $a^*$ and $b$, namely $\nu^{(b,a^*)}_{a^*} = \nu_b$, $\nu^{(b,a^*)}_{b} = \nu_{a^*}$, and $\nu^{(b,a^*)}_{a} = \nu_a$ for $a \in [n] - \{a^*,b\}$.  Then, if $\Alg$ is $\delta$-correct, one has that for any $\eta \in (0,1/4)$
	\begin{eqnarray}
	 \frac{1}{2}\left\{\Pr_{\nu,\Alg}\left[N_{b}(T) > \tau(\eta)\right] +  \Pr_{\nu^{(b,a^*)},\Alg}\left[N_{a^*}(T) > \tau(\eta)\right] \right\} \ge \eta - \delta~,
	\end{eqnarray}
	where $\tau(\eta) = \frac{ 1}{\KL(\nu_{a^*},\nu_{b}) +\KL(\nu_{b},\nu_{a^*})}\log(1/4\eta)$
	\end{proposition}
	This bound implies that, if an instance $\overline{\nu}$ is drawn uniformly from $\{\nu,\nu^{(b,a^*)}\}$, then any $\delta$-correct algorithm has to pull the suboptimal arm, namely the distribution $\nu_b$, at least $\tau(\eta)$ times on average (over the draw of $\overline{\nu})$, with probability $\eta - \delta$. Proving this proposition requires choosing an appropriate simulator. To this end, fix a $\tau \in \N$, and let $\Sim$ map $\Trs$ to $\Trshat$ such that, 
	\begin{eqnarray}
	\Sim: \Xhat_{[s,a]} \mapsfrom \begin{cases} X_{[s,a]} & a \ne a^*,b \\
	X_{[s,a]} &  a \in \{a^*,b\}, s \le \tau \\
	\overset{iid}{\sim} \nu_{a^*} & a \in \{a^*,b\}, s > \tau 
	\end{cases}
	\end{eqnarray}
	where for $s > \tau$ and $a \in \{a^*,b\}$, the $\Xhat_{[s,a]} \overset{iid}{\sim} \nu_{a^*}$ means that the samples are taken independently of everything else (in particular, independent of $X_{[s,a^*]}$ and $X_{[s,b]}$), using internal randomness $\xi_{\Sim}$. We emphasize $\Sim$ depends crucially on $\nu$, $a^*$, and $b$.

	Note that the only entries of $\Trshat$ whose distribution differs under $\Sim(\nu)$ and $\Sim(\nu^{(b,a^*)})$ are just the first $\tau$ entries from arms $a^*$ and $b$, namely $\{\Xhat_{s,a}\}_{1 \le s \le \tau, a \in \{a^*,b\}}$. Hence, by a data-processing inequality
	\begin{eqnarray}
	\KL_{\Alg}(\Sim(\nu),\Sim(\nu^{(b,a^*)}) \le \tau\{\KL(\nu_{a^*},\nu_{b}) +\KL(\nu_{b},\nu_{a^*})\}\end{eqnarray}
	Using the notation of Proposition~\ref{Prop:SimLeCam}, let $\nu^{(1)} = \nu$, $\nu^{(2)} = \nu^{(b,a^*)}$, let $W_{1} := \{N_{b}(T) \le \tau\}$ and $W_2 := \{N_{a^*}(T) \le \tau\}$ (i.e, under $\nu^{(i)}$ and $W_i$, you sample the suboptimal arm no greater than $\tau$ times). Now, Proposition~\ref{cor:2arm} now follows immediately from Proposition~\ref{Prop:SimLeCam}, elementary manipulations, and the following claim:
	\begin{claim}\label{TruthfulClaim} For $\nu^{(i)}$ and $W_i$ defined above, $\Sim$ is truthful on $W_i$ under $\nu^{(i)}$.
	\end{claim} 

	\begin{proof}[Proof of Claim~\ref{TruthfulClaim}] The samples $\Xhat_{[s,a]}$ and $X_{[s,a]}$ have the sample distribution under $\nu^{(i)}$ and $\Sim(\nu^{(i)})$ for $a \notin \{a^*,b\}$ and $s \le \tau$, by construction. Moreover, the samples $\Xhat_{[s,a^*]}$ and $\Xhat_{[s,b]}$ for $s > \tau$ are also i.i.d draws from $\nu_{a^*}$, so they have the same distribution as the samples $X_{[s,a^*]}$ and $X_{[s,b]}$ under $\nu^{(1)}$ and $\nu^{(2)}$ respectively. Thus, the only samples whose distributions are changed by the simulator are the samples $\Xhat_{[s,b]}$ under $\nu^{(1)}$ and $\Xhat_{[s,b]}$ under $\nu^{(2)}$, respectively, which $\Alg$ never accesses under  under $W_1$ and $W_2$, respectively. 
	\end{proof}

\subsection{Proving Theorem~\ref{Thm:PermLB} from Proposition~\ref{cor:2arm}\label{sec:proof_of_perm_lb}}
	Theorem~\ref{Thm:PermLB} can be proven directly using the machinery established thus far. However, we will introduce a reduction to ``symmetric algorithms'' which will both expedite the proof of the Theorem~\ref{Thm:PermLB}, and come in handy for additional bounds as well. For a transcript $\Trs$, let $\pi(\Trs)$ denote the transcript $\pi(\Trs)_{a,s} = \Trs_{\pi(a),s}$, and $\Pr_{\Alg,\Trs}$ denote probability taken w.r.t. the randomness of $\Alg$ acting on the fixed (deterministic) transcript $\Trs$. For any subset $S \subset [n]$, we take $\pi(S) := \{\pi(a): a\in S\}$. 
	\begin{definition}[Symmetric Algorithm]
	We say that an algorithm $\Alg$ is \emph{symmetric} if the distribution of its sampling sequence and output commutes with permutations. That is, for any permutation $\pi$, transcript $\Trs$, sequence of actions $(A_1,A_2,\dots)$, and output $\widehat{S} \subset [n]$,
	\begin{multline}
	\Pr_{\Alg,\Trs}\left[(a_1,a_2,\dots,a_T,\widehat{S}) = (A_1,A_2,\dots,A_T,S)\right] \\
	= \Pr_{\Alg,\pi(\Trs)}\left[(a_1,a_2,\dots,a_T,\widehat{S}) = (\pi(A_1),\pi(A_2),\dots,\pi(A_T),\pi(S))\right]
	\end{multline}
	\end{definition}
	In particular, if $\Alg$ is symmetric, then $\Pr_{\nu,\Alg}[N_{b}(\widetilde{T}) \ge \tau] = \Pr_{\pi(\nu),\Alg}[N_{\pi(b)}(\widetilde{T}) \ge \tau]$ for all $b \in [n]$, $\pi \in \bS_n$, and $\{\mathcal{F}_t\}$-measurable stopping time $\widetilde{T}$.  The following lemma reduces lower bounds on average complexity over permutations to lower bounds on a single instance for a symmetric algorithm (see Section~\ref{SymProof} for proof and discussion): 


	\begin{lemma}[Algorithm Symmetrization]\label{SymmetryLemma} Let $\Alg$ be a $\delta$-correct algorithm over $\bS_n(\nu)$. Then there exists a \emph{symmetric} algorithm $\Alg^{\bS_n}$, which is also $\delta$ correct over $\bS_n(\nu)$, and such that, for any $\{\mathcal{F}_t\}$-measurable stopping time $\widetilde{T}$ (in particular, $\widetilde{T} = T$)
	\begin{eqnarray}
	\Pr_{\nu,\Alg^{\bS_n}}[N_{b}(\widetilde{T}) \ge \tau] =  \Pr_{\pi \sim \bS_n}\Pr_{\pi(\nu),\Alg}[N_{\pi(b)}(\widetilde{T}) \ge \tau]
	\end{eqnarray}
	\end{lemma}
	Now, we are ready to prove Theorem~\ref{Thm:PermLB}
	\begin{proof}[Proof of Theorem~\ref{Thm:PermLB}]
	We first establish~\ref{Eq:PermutationIndividual} for $\delta$-correct symmetric algorithms, and use Lemma~\ref{SymmetryLemma} to extend to all $\delta$-correct algorithms. Again, let $\nu^{(b,a^*)}$ be the instance obtained by swapping $a^*$ and $b$, and let $\pi_b$ be the permutation yielding $\pi_b(\nu) = \nu^{(b,a^*)}$. Adopt the shorthand $\tau_b(\eta) = \tau_b \cdot \log(1/4\eta)$. Then assuming $\Alg$ is symmetric and noting that $\pi_b(a^*) = b$, we have
	\begin{eqnarray*}
	\Pr_{\pi \sim \bS_n}\Pr_{\pi(\nu),\Alg}[N_{\pi(b)}(T) > \tau_{b}(\eta)] &\overset{(i)}{=}& \Pr_{\nu,\Alg}[N_{b}(T) > \tau_{b}(\eta)] \\
	&\overset{(ii)}{=}& \frac{1}{2}\left\{\Pr_{\nu,\Alg}[N_{b}(T) > \tau_{b}(\eta)] + \Pr_{\pi_b(\nu),\Alg}[N_{\pi_b(b)}(T) > \tau_{b}(\eta)]\right\}\\
	&\overset{(iii)}{=}& \frac{1}{2}\left\{\Pr_{\nu,\Alg}[N_{b}(T) > \tau_{b}(\eta)] + \Pr_{\nu^{(b,a^*)},\Alg}[N_{a^*}(T) > \tau_{b}(\eta)]\right\}~,
	\end{eqnarray*}
	where $(i)$ and $(ii)$ follow from the definition of symmetric algorithms, $(iii)$ follows from how we defined the permutation $\pi_b$. Applying Proposition~\ref{cor:2arm}, the above is at most $\eta - \delta$. Next, we show that Equation~\ref{Eq:PermutationIndividual} implies Equation~\ref{Eq:PermutationTotal}. This part of the proof need not invoke that $\Alg$ is symmetric. Applying Markov's inequality  Equation~\ref{Eq:PermutationIndividual} implies that $\Exp_{\pi \sim \bS_n}\Exp_{\pi(\nu),\Alg} \ge \log(1/4\eta)(\eta-\delta)\tau_b$. Hence, 
	\begin{eqnarray*}
	&& \Exp_{\pi \sim \bS_n}\Exp_{\pi(\nu),\Alg}[T] \enspace = \enspace \Exp_{\pi \sim \bS_n}\Exp_{\pi(\nu),\Alg}[\sum_{b \in [n]}N_b(T)]\enspace = \enspace \Exp_{\pi \sim \bS_n}\Exp_{\pi(\nu),\Alg}[\sum_{b \in [n]}N_{\pi(b)}(T)] \\
	&& \enspace \ge \enspace \Exp_{\pi \sim \bS_n}\Exp_{\pi(\nu),\Alg}[\sum_{b \ne a^*}N_{\pi(b)}(T)] \enspace \ge \enspace \log(1/4\eta)(\eta-\delta) \sum_{b \ne a^*} \tau_b
	\end{eqnarray*}
	\end{proof}

%% file: conclusion.tex
\section{Conclusion}\label{sec:conclusion}
In the pursuit of understanding the fundamental limits of adaptive sampling in the presence of side knowledge about the problem (e.g. the means of the actions are known to lie in a known set), we unearthed fundamental limitations of the existing machinery (i.e., change of measure and Fano's inequality). In response, we developed a new framework for analyzing adaptive sampling problems -- the Simulator -- and applied it to the particular adaptive sampling problem of multi-armed bandits to obtain state-of-the-art lower bounds. New insights from these lower bounds led directly to formulating a new algorithm for the TOP-K problem that is state-of-the-art in both theory and practice. Armed with the tools and demonstration of their use on a simple problem, we are convinced that this recipe can be used to produce future successes for more structured adaptive sampling problems, the true goal of this work.

%% file: upper_bound_proof.tex
\section{Upper Bound Proof\label{Sec:UpperBoundProof}}
\begin{proof}[Proof of Theorem~\ref{Thm:LUCBplus}]
Observe that if $\TOP_t \neq [k]$ then there is at least one arm from $[k]$ in $\TOP_t^c$. 
Since we play the arm $l_t \in TOP_t^c$ with the largest upper-confidence-bound, the arms in $[k] \cap \TOP_t^c$ will eventually rise to the top. 
A mirror image of this process is also happening in the top empirical arms: if $\TOP_t \neq [k]$ then there is at least one arm from $[k]^c$ in $\TOP_t$ and the arms in $[k]^c \cap \TOP_t$ will eventually fall to the bottom since the arm $h_t \in \TOP_t$ with the lowest-confidence-bound is played.
Thus, for some sufficently large $t$, the arms in $[k] \cap \TOP_t^c$ and $[k]^c \cap \TOP_t$ start to concentrate around the gap between the $k$th and $(k+1)$th arm.  
Because we play an arm from each side of the gap at each time, $h_t$ and $l_t$, the empirical means eventually converge to their true means revealing the true ordering. 

\noindent\textbf{Preliminaries}\\
We will use the following quantities throughout the proof. 
Let $U(t,\delta) \propto \sqrt{\tfrac{1}{t} \log(\log(t)/\delta)}$ such that $\max\{ \P( \bigcup_{t=1}^\infty \left\{ \widehat{\mu}_{i,t} - \mu_i \geq U(t,\delta) \right\} ), \P( \bigcup_{t=1}^\infty \left\{ \widehat{\mu}_{i,t} - \mu_i \leq -U(t,\delta) \right\} ) \} \leq \delta$ and $U(\cdot,\cdot)$ is decreasing in its first and second arguments (see \citet[Lemma 1]{jamieson2014lil} or \citet[Theorem 8]{kaufmann2014complexity} for an explicit expression).
For any $i\in [n]$ define
\begin{align*}
\calE_i = \begin{cases} \{ \widehat{\mu}_{i,t} - \mu_i \geq  -U(t,\tfrac{\delta}{2k}) \} & \text{ if  } i \in \{1,\dots,k\}  \\
\{ \widehat{\mu}_{i,t} - \mu_i \leq  U(t,\tfrac{\delta}{2(n-k)}) \} & \text{ if  } i \in \{k+1,\dots,n\}. \end{cases}
\end{align*}
In what follows assume that $\mathcal{E}_i$ hold for all $i \in [n]$ since, by the definition of $U(t,\delta)$,
\begin{align}\label{eqn:events_hold}
\P\left( \bigcup_{i=1}^n \calE_i^c \right) \leq \sum_{i=1}^k \frac{\delta}{2k} + \sum_{i=k+1}^n \frac{\delta}{2(n-k)} \leq \delta.
\end{align}
For any $j \in [k]^c$ define the random variable 
\begin{align} \label{eqn:rho_j}
\rho_j =  \sup\{ \rho>0 : \widehat{\mu}_{j,t}-\mu_j < U(t,\tfrac{\rho \delta}{2k}) \ \forall t  \}
\end{align}
and the quantity $\tau_j = \min\{ t : U(t,\tfrac{\rho_j \delta}{2k}) < \Delta_j/2 \}$.
Note that on the event $\mathcal{E}_j$ we have $\rho_j \geq \frac{k}{n-k}$ which guarantees that $\tau_j$ is finite, but we will show that $\rho_j$ is typically actually $\Omega(1)$.
For any $i \in [k]$ define $\tau_i = \min\{ t : U(t,\tfrac{\delta}{2(n-k)}) < \Delta_i/2 \}$, and note that
on $\mathcal{E}_i$, we have $\widehat{\mu}_{i,t}-\mu_i \geq -U(t,\tfrac{\delta}{2k}) \geq -U(t,\tfrac{\delta}{2(n-k)})$.
From these definitions, we conclude that 
\begin{align} \label{eqn:tau_j}
\widehat{\mu}_{j,t}-\mu_j \leq \Delta_j/2 \quad \forall t \geq \tau_j, j \in [k]^c
\quad \text{ and } \quad
\widehat{\mu}_{i,t}-\mu_i \overset{\mathcal{E}_i}{\geq} -\Delta_i/2 \quad \forall t \geq \tau_i, i \in [k].
\end{align}
We leave the $\tau_i$ random variables unspecified for now but will later upper bound their sum.

\noindent\textbf{Step 0: Correctness}\\
Suppose $\TOP_\tau \neq [k]$. Then there exists an $i \in \TOP_\tau \cap [k]^c$ and $j \in \TOP_\tau^c \cap [k]$ such that 
\begin{align*}
\mu_i \  \overset{(i)}{\geq} \ 
\widehat{\mu}_{i,N_i(t)} - U(N_i(t),\tfrac{\delta}{2(n-k)}) \overset{(ii)}{>}  \widehat{\mu}_{j,N_j(t)} + U(N_j(t),\tfrac{\delta}{2k}) \overset{(iii)}{\geq} \mu_j
\end{align*}
where $(i)$ and $(iii)$ employ Equation~\ref{eqn:events_hold}, and $(ii)$ holds by assumption because of the stopping time $\tau$. 
This display implies $\mu_j < \mu_i$, a contradiction, since the means in $[k]$ are strictly greater than those in $[k]^c$.

\noindent\textbf{Step 1: Decomposition of sample complexity}\\
If $\tau$ is the stopping time of the algorithm, then
\begin{align*}
\tau &= \sum_{t=1}^\tau \1\{ [k] = \TOP_t \} + \1\{ [k] \neq \TOP_t \} \\
&= \sum_{t=1}^\tau \1\{ [k] = \TOP_t \} + \1\{ [k] \neq \TOP_t, l_t \notin [k] \} + \1\{ l_t \in [k] \} \\
&= \sum_{t=1}^\tau \1\{ [k] = \TOP_t \} + \1\{ [k] \neq \TOP_t, l_t \notin [k] \} + \1\{ l_t \in [k], h_t \notin [k] \} + \1\{ l_t \in [k], h_t \in [k] \} 
\end{align*}
so we will bound each of these last four sums individually.\\

\noindent\textbf{Step 2: Bound $\sum_{t=1}^\tau \1\{ [k] \neq \TOP_t, l_t \notin [k] \}$}\\
\noindent Note that
\begin{align*}
[k] \neq \TOP_t , l_t \notin [k] \implies \exists i \in [k], j \in[k]^c : \widehat{\mu}_{j,N_j(t)} + U(N_j(t),\tfrac{\delta}{2k}) \geq \widehat{\mu}_{i,N_i(t)} + U(N_i(t),\tfrac{\delta}{2k}).
\end{align*}
By the definition of $\rho_j$ in Equation~\ref{eqn:rho_j} and the above implication, 
\begin{align*}
\mu_j + 2 U(N_j(t),\tfrac{\delta\rho_j}{2k}) \geq 
\widehat{\mu}_{j,N_j(t)} + U(N_j(t),\tfrac{\delta}{2k}) &\geq \widehat{\mu}_{i,N_i(t)} + U(N_i(t),\tfrac{\delta}{2k}) 
\overset{\mathcal{E}_i}{\geq} \mu_i \geq \mu_k
\end{align*}
where the second-to-last inequality holds on event $\mathcal{E}_i$. 
Recalling that $\Delta_j = \mu_k - \mu_j$ we use the above logic to conclude that
\begin{align*}
\sum_{t=1}^\tau &\1\{ [k] \neq \TOP_t, l_t \notin [k] \} = \sum_{j=k+1}^n \sum_{t=1}^\tau \1\{ [k] \neq \TOP_t, l_t = j \} \\
&\leq \sum_{j=k+1}^n \sum_{t=1}^\tau \1\{ l_t = j, \exists i \in [k] : \widehat{\mu}_{j,N_j(t)} + U(N_j(t),\tfrac{\delta}{2k}) \geq \widehat{\mu}_{i,N_i(t)} + U(N_i(t),\tfrac{\delta}{2k}) \} \\
&\leq \sum_{j=k+1}^n \sum_{t=1}^\tau \1\{ l_t = j, U(N_j(t),\tfrac{\delta\rho_j}{2k}) \geq \Delta_j/2 \}\\
&\leq \sum_{j=k+1}^n \sum_{t=1}^\infty \1\{ U(t,\tfrac{\delta\rho_j}{2k}) \geq \Delta_j/2 \} \leq \sum_{j=k+1}^n \tau_j
\end{align*}
by the definition of $\tau_j$.\\


\noindent\textbf{Step 3: Bound $\sum_{t=1}^\tau \1\{l_t \in [k], h_t \notin [k]\}$}\\
Note that 
\begin{align*}
l_t \in [k], h_t \notin [k] \implies \exists i \in [k], j \in [k]^c : \widehat{\mu}_{j,N_j(t)} > \widehat{\mu}_{i,N_i(t)}.
\end{align*}
By the definition of $\rho_j$ in Equation~\ref{eqn:rho_j} and the above implication, 
\begin{align*}
\mu_j + U(N_j(t),\tfrac{\delta\rho_j}{2k}) \geq 
\widehat{\mu}_{j,N_j(t)} &\geq \widehat{\mu}_{i,N_i(t)}  
\overset{\mathcal{E}_i}{\geq} \mu_i - U(N_i(t),\tfrac{\delta}{2k}) \geq \mu_i - U(N_i(t),\tfrac{\delta}{2(n-k)})
\end{align*}
where the second-to-last inequality holds on event $\mathcal{E}_i$, and the last from $- U(N_i(t),\tfrac{\delta}{2k}) \geq - U(N_i(t),\tfrac{\delta}{2(n-k)})$. 
Now, since $a + b \geq \Delta \implies a \geq \Delta/2$ or $b \geq \Delta/2$ we have
\begin{align*}
\sum_{t=1}^\tau &\1\{ l_t \in [k], h_t \notin [k] \} \leq \sum_{i=1}^k \sum_{j=k+1}^n \sum_{t=1}^\tau \1\{ h_t = j, l_t = i, U(N_j(t),\tfrac{\delta\rho_j}{2k}) + U(N_i(t),\tfrac{\delta}{2(n-k)}) \geq \mu_i - \mu_j \} \\
&\leq \sum_{i=1}^k \sum_{j=k+1}^n \sum_{t=1}^\tau \1\{ h_t = j, l_t = i, U(N_j(t),\tfrac{\delta\rho_j}{2k}) + U(N_i(t),\tfrac{\delta}{2(n-k)}) \geq (\mu_i-\mu_j) \} \\
&\leq \sum_{i=1}^k \sum_{j=k+1}^n \sum_{t=1}^\tau \1\{ h_t = j, l_t = i, U(N_j(t),\tfrac{\delta\rho_j}{2k}) \geq (\mu_i-\mu_j)/2 \} \\
&\quad+ \sum_{i=1}^k \sum_{j=k+1}^n \sum_{t=1}^\tau  \1\{ h_t = j, l_t = i, U(N_i(t),\tfrac{\delta}{2(n-k)}) \geq (\mu_i-\mu_j)/2 \} \\
&\leq \sum_{i=1}^k \sum_{j=k+1}^n \sum_{t=1}^\tau \1\{ h_t = j, l_t = i, U(N_j(t),\tfrac{\delta\rho_j}{2k}) \geq \Delta_j/2 \} \\
&\quad+ \sum_{i=1}^k \sum_{j=k+1}^n \sum_{t=1}^\tau  \1\{ h_t = j, l_t = i, U(N_i(t),\tfrac{\delta}{2(n-k)}) \geq \Delta_i/2 \} \\
&\leq \sum_{j=k+1}^n \sum_{t=1}^\tau \1\{ h_t = j, U(N_j(t),\tfrac{\delta\rho_j}{2k}) \geq \Delta_j/2 \} + \sum_{i=1}^k \sum_{t=1}^\tau  \1\{  l_t = i, U(N_i(t),\tfrac{\delta}{2(n-k)}) \geq \Delta_i/2 \} \\
&\leq \sum_{j=k+1}^n \sum_{t=1}^\infty \1\{ U(t,\tfrac{\delta\rho_j}{2k}) \geq \Delta_j/2 \} + \sum_{i=1}^k \sum_{t=1}^\infty  \1\{  U(t,\tfrac{\delta}{2(n-k)}) \geq \Delta_i/2 \} \\
&\leq \sum_{j=k+1}^n \tau_j + \sum_{i=1}^k \tau_i.
\end{align*}

\noindent\textbf{Step 4: Bound $\sum_{t=1}^\tau \1\{  l_t \in [k], h_t \in [k]  \}$}\\
Note that 
\begin{align*}
 l_t \in [k], h_t \in [k]  \implies \exists i \in [k], j \in [k]^c : \widehat{\mu}_{j,N_j(t)} - U(N_j(t),\tfrac{\delta}{2(n-k)}) \geq \widehat{\mu}_{i,N_i(t)} - U(N_i(t),\tfrac{\delta}{2(n-k)}).
 \end{align*}
By the events $\mathcal{E}_j,\mathcal{E}_i$ and the above implication we have
\begin{align*}
\mu_{k+1} \geq \mu_j \overset{\mathcal{E}_j}{\geq} 
\widehat{\mu}_{j,N_j(t)} -U(N_j(t),\tfrac{\delta}{2(n-k)}) &\geq \widehat{\mu}_{i,N_i(t)}  - U(N_i(t),\tfrac{\delta}{2(n-k)})
\overset{\mathcal{E}_i}{\geq} \mu_i - 2 U(N_i(t),\tfrac{\delta}{2(n-k)})
\end{align*}
since $-U(t,\tfrac{\delta}{2k}) \geq -U(t,\tfrac{\delta}{2(n-k)})$.
Thus,
\begin{align*}
 l_t \in [k], h_t \in [k]  \implies U(N_{h_t}(t),\tfrac{\delta}{2(n-k)}) \geq \Delta_{h_t}/2
 \end{align*}
so that
\begin{align*}
\sum_{t=1}^\tau \1\{ l_t \in [k], h_t \in [k] \} &\leq \sum_{i = 1}^k \sum_{t=1}^\tau \1\{ h_t = i,U(N_{i}(t),\tfrac{\delta}{2(n-k)}) \geq \Delta_{i}/2 \} \\
&\leq \sum_{i = 1}^k \sum_{t=1}^\infty \1\{ U(t,\tfrac{\delta}{2(n-k)}) \geq \Delta_{i}/2 \} \leq \sum_{i=1}^k \tau_i
\end{align*}

\noindent\textbf{Step 5: Bound $\sum_{t=1}^\tau \1\{  [k] = \TOP_t  \}$}\\
Implicit in the event that $\{ [k] = \TOP_t \}$ is that the game has not ended yet, or that $t \leq \tau$.
\begin{align*}
[k] = \TOP_t \implies \exists i \in [k], j \in [k]^c : \widehat{\mu}_{i,N_i(t)} - U(N_i(t),\tfrac{\delta}{2(n-k)}) < \widehat{\mu}_{j,N_j(t)} + U(N_j(t),\tfrac{\delta}{2k}).
 \end{align*}
By $\mathcal{E}_i$ and the definition of $\rho_j$ in Equation~\ref{eqn:rho_j} and the above implication, 
\begin{align*}
{\mu}_{i} - 2U(N_i(t),\tfrac{\delta}{2(n-k)}) \overset{\mathcal{E}_i}{\leq} \widehat{\mu}_{i,N_i(t)} - U(N_i(t),\tfrac{\delta}{2(n-k)}) < \widehat{\mu}_{j,N_j(t)} + U(N_j(t),\tfrac{\delta}{2k}) \leq \mu_j + 2 U(N_j(t),\tfrac{\delta \rho_j}{2k}).
\end{align*}
Now, by an identical argument to Step 3 above, we have
\begin{align*}
\sum_{t=1}^\tau &\1\{ [k]=\TOP_t \} \leq \sum_{i=1}^k \sum_{j=k+1}^n \sum_{t=1}^\tau \1\{ h_t = i, l_t = j, U(N_j(t),\tfrac{\delta\rho_j}{2k}) + U(N_i(t),\tfrac{\delta}{2(n-k)}) \geq (\mu_i - \mu_j)/2 \} \\
&\leq \sum_{j=k+1}^n \sum_{t=1}^\infty \1\{ U(t,\tfrac{\delta\rho_j}{2k}) \geq \Delta_j/4 \} + \sum_{i=1}^k \sum_{t=1}^\infty  \1\{  U(t,\tfrac{\delta}{2(n-k)}) \geq \Delta_i/4 \} \\
&\leq \sum_{j=k+1}^n \tau_j + \sum_{i=1}^k \tau_i.
\end{align*}

\noindent\textbf{Step 6: Counting the number of measurements}\\
Putting all the pieces together, we conclude that the total number of measurements taken at the stopping time $\tau$ is bounded by $3\sum_{i=1}^n \tau_i$. Recall that $\tau_j$ is a random variable because $\rho_j$ for $j \in [k]^c$ are random variables. 
Recalling the definitions of $\tau_j$ preceding Equation~\ref{eqn:tau_j}, we note that
\begin{align*}
\min\{ t : U(t,s) < \Delta/2 \} \leq c \Delta^{-2} \log( \log(\Delta^{-2})/s )
\end{align*}
for some universal constant $c$.
For $i \in [k]$ this means $\tau_i = \min\{ t : U(t,\tfrac{\delta}{2(n-k)}) < \Delta_i/2 \} \leq c \Delta_i^{-2} \log( 2(n-k)\log(\Delta_i^{-2})/\delta )$.
For $j \in [k]^c$ we have 
\begin{align*}
\tau_j = \min\{ t : U(t,\tfrac{\rho_j \delta}{2k}) < \Delta_j/4 \} \leq c \Delta_j^{-2} \log( 2k\log(\Delta_j^{-2})/\delta) + c \Delta_j^{-2} \log( 1/\rho_j).
\end{align*}
By the definition of $U( \cdot, \cdot)$ and $\rho_j$ we have that $\P(\rho_j \leq \rho) \leq \frac{\rho \delta}{2k} < \rho$, so reparameterizing with $\rho = \exp(-s \Delta_j^2)$
\begin{align*}
\P( \Delta_j^{-2} \log( 1/\rho_j) \geq s ) \leq \exp( - s \Delta_j^2 /2 )
\end{align*}
which implies $\Delta_j^{-2} \log( 1/\rho_j)$ is an independent sub-exponential random variable.
Using standard techniques for sums of independent random variables (see \cite[Lemma 4]{jamieson2014lil} for an identical calculation) we observe that with probability at least $1-\delta$
\begin{align*}
\sum_{j=k+1}^n \Delta_j^{-2} \log( \tfrac{ 1}{\rho_j} ) \leq \sum_{j=k+1}^n c' \Delta_j^{-2} \log(1/\delta) 
\end{align*}
for some universal constant $c'$. Combining the contributions of the deterministic components of $\tau_i$ and $\tau_j$ obtains the result.

\end{proof}

%% file: permutation_upper_bounds.tex

\subsection{Upper Bounds for Permutations\label{sec:permu_upper_bounds}}
In this section, we present a nearly-matching upper bound for permutations (Theorem~\ref{UpperBound-Permutations}). For simplicity, we consider the setting where each measure $\nu_a$ is $1$-subGaussian, and has mean $\mu_a$. We let $\mu_{(1)} > \mu_{(2)} \ge \dots \ge \mu_{(n)}$, denote the sorted means, and set $\Delta_i = \mu_{(1)} - \mu_{(i)}$. 
\begin{theorem}\label{UpperBound-Permutations}
In the setting given above, there exists a $\delta$- algorithm $\Alg$ which, given knowledge of the means $\mu_{(1)}$ and $\mu_{(2)}$, returns the top arm with expected sample complexity 
\begin{eqnarray}
\Exp_{\nu,\Alg}[T] \lesssim \frac{\log (1/\delta)}{\Delta_2^2} + \sum_{i=1}^n \frac{\log\log(\min \{n,\Delta_i^{-1}\})}{\Delta_i^2}
\end{eqnarray}
\end{theorem}
We remark that this upper bound matches our lower bound up to the doubly-logarithmic factor $\log\log(\min \{n,\Delta_i^{-1}\})$. We believe that one could remove this factor when the means are known up to a permutation, though closing this small gap is beyond the scope of this work. To prove the above theorem, we combine the following Lemma with the best-arm algorithm from \cite{chen2015optimal}:
\begin{proposition} Suppose that for each $\delta$, there exists an (unconstrained) $\MAB$ algorithm $\Alg_{\delta}$ which is $\delta$-correct for 1-subGaussian distributions with unconstrained means, and satisfies $\Exp_{\nu,\Alg_{\delta}}[T] \le H_1(\nu) + H_2(\nu)\log(1/\delta)$. Then, there exists an an $\MAB$ algorithm which, give knowledge of the the best mean $\mu_1$ and the second best mean $\mu_2$, satisfies
\begin{eqnarray}
\Exp_{\nu,\Alg}[T] \lesssim \frac{\log (1/\delta)}{\Delta_2^2} + H_1(\nu) + H_2(\nu)
\end{eqnarray} 
\end{proposition}
\begin{proof}
Fix constants $c_1$ and $c_2$ to be chosen later The algorithm proceeds in stages: at round $k$, set $\delta_k = 10^{-k}$, and run $\Alg_{\delta_k}$ to get an estimate $\hat{a}_k$ of the best arm. Then, sample $\hat{a}_k$ $\frac{c_1}{\Delta_2^2}\log(c_2k^2/\delta)$ times to get an estimate $\widehat{\mu}^k$, and return $\widehat{a} = \hat{a}_k$ if $\widehat{\mu}^k > \mu_1 - \Delta_2/2$. By a standard Chernoff bound, we can choose $c_1$ so that $\widehat{\mu}_k$ satisfies the following 
\begin{eqnarray*}
\Pr(\widehat{\mu}^k > \mu_1 - \Delta_2/2 \big{|} \hat{a}_k = a^*) \ge 1 - 2\delta/c_2k^2 &\text{and}& \Pr(\widehat{\mu}^k > \mu_1 - \Delta_2/2 \big{|} \hat{a}_k \ne a^*) \le 2\delta/c_2k^2
\end{eqnarray*}
Hence,
\begin{eqnarray}
\Pr(\widehat{a} \ne a^*) &\le& \sum_{k=1}^{\infty} \Pr(\{\widehat{\mu}^k \ne a^*\} \wedge \{\widehat{\mu}^k > \mu_1 - \Delta_2/2\})\\
&\le& \sum_{k=1}^{\infty} \Pr( \{\widehat{\mu}^k > \mu_1 - \Delta_2/2\} \big{|} \{\widehat{\mu}^k \ne a^*\})  \le \frac{2\delta}{c_2} \sum_{k=1}^{\infty} k^{-2} = \frac{\pi^2}{3c_2}
\end{eqnarray}
Hence, choosing $c_2 = 3/\pi^2$ ensures that $\Alg$ is $\delta$-correct. Moreover, we can bound 
\begin{eqnarray}
\Exp_{\nu,\Alg}[T] &\le& \sum_{k = 1}^{\infty} \Pr(E_{k-1})*\{\frac{c_1}{\Delta_2^2}\log(c_2k^2/\delta)  + \Exp_{\nu,\Alg_{\delta_k}}[T]\}
\end{eqnarray}
where $E_{k-1}$ is the event that the algorithm has not terminated by stage $k-1$. Note that if the algorithm has not terminated at a stage $j$, then it is not the case that $\hat{a}_j = a^*$ and $\{\widehat{\mu}^j > \mu_1 - \Delta_2/2\})$. By a union bound, the probability that these two events don't occur is at most $1 - \delta_k - \frac{2\delta}{c_2k^2} \le 1 - (\delta_k + 2\delta/c_2) \le 1/2$. Hence, bounding $\Exp_{\nu,\Alg_{10^{-k}}}[T] \le H_1(\nu) + k H_2(\nu)\log 10$, and using independence of the rounds have
\begin{eqnarray}
\Exp_{\nu,\Alg}[T] &\le& \sum_{k = 1}^{\infty} 2^{1-k}*\{\frac{c_1}{\Delta_2^2}\log(c_2k^2/\delta)  + H_{1}(\nu) + kH_2(\nu)\log 10)\}\\
&\lesssim& \frac{\log (1/\delta)}{\Delta_2^2} + H_1(\nu) + H_2(\nu)
\end{eqnarray}

\end{proof}

%% file: two_hypothesis_proofs.tex

\section{Proofs for Section~\ref{Sec:Permutations}\label{App:Permutations}}
\subsection{Proof of Propostion~\ref{Prop:SimLeCam}\label{sec:simlecamproof}}

	First, by combining Pinkser's Inequality with the data processing inequality \citep{tsybakov2009introduction}, we arive at an elementary bound that controls $\TV$ between runs of an algorithm on simulated instances:

	\begin{lemma}[Pinkser's Inequality]\label{Pinsker} Let $\nu^{(1)}$ and $\nu^{(2)}$ be two measures. Then for any simulator $\Sim$,
	\begin{eqnarray}\label{PinkserEq}
	\sup_{E \in \mathcal{F}_{T}}\left|\Pr_{\Sim(\nu^{(1)}),\Alg}(E) - \Pr_{\Sim(\nu^{(2)}),\Alg}(E)\right|
	&\le& \TV_{\Alg}\left(\Sim(\nu^{(1)}),\Sim(\nu^{(2)})\right)\\
	&\le&Q\left(\KL_{\Alg}\left(\Sim(\nu^{(1)}),\Sim(\nu^{(2)})\right)\right)
	\end{eqnarray}
	Where $Q(\beta) = \min\left\{1 - \frac{1}{2}e^{-\beta},  \sqrt{\beta/2}\right\}$.
	\end{lemma}
	Note here that we only consider events $E \in \mathcal{F}_T$, which only depend on the samples $X_{a_1,1},\dots,X_{a_T,t}$ collected from the modified $\widehat{\Trs}$. Now we can prove our result.
	\begin{proof}[Proof of Proposition~\ref{Prop:SimLeCam}]
	By the triangle inequality
	\begin{multline}
	|\Pr_{\nu^{(1)},\Alg}(E) - \Pr_{\nu^{(2)},\Alg}(E)| \\\ \le |\Pr_{\Sim(\nu^{(1))},\Alg}(E) - \Pr_{\Sim(\nu^{(2)}),\Alg}(E)| + \sum_{i=1}^2 |\Pr_{\Sim(\nu^{(i))}),\Alg}(E) - \Pr_{\nu^{(i))},\Alg}(E)| 
	\end{multline}
	We can expand
	\begin{eqnarray*}
	&&\Pr_{\Sim(\nu^{(i)}),\Alg}(E) - \Pr_{\nu^{(i)},\Alg}(E)  \\
	 &=& \Pr_{\Sim(\nu^{(i)}),\Alg}(E \wedge W_i) + \Pr_{\Sim(\nu^{(i))}),\Alg}(E \wedge W_i^c) - (\Pr_{\nu^{(i)},\Alg}(E \wedge W_i) + \Pr_{\nu^{(i)},\Alg}(E \wedge W_i^c))\\
	&=&  \Pr_{\Sim(\nu^{(i)}),\Alg}(E \wedge W_i^c) - \Pr_{\nu^{(i)},\Alg}(E \wedge W_i^c)
	\end{eqnarray*}
	where $\Pr_{\Sim(\nu^{(i)}),\Alg}(E \wedge W_i) = \Pr_{\nu^{(i)},\Alg}(E \wedge W_i)$ as $W_i$ is truthful for $\nu^{(i)}$. Thus,
	\begin{eqnarray*}
	|\Pr_{\Sim(\nu^{(i)}),\Alg}(E) - \Pr_{\nu^{(i)},\Alg}(E)| &=&|\Pr_{\Sim(\nu^{(i)}),\Alg}(E \wedge W_i^c) - \Pr_{\nu^{(i)},\Alg}(E \wedge W_i^c)|\\
	&\overset{(i)}{\le}& \max\{\Pr_{\Sim(\nu^{(i)}),\Alg}(E \wedge W_i^c), \Pr_{\nu^{(i)},\Alg}(E \wedge W_i^c)\}\\
	&\overset{(ii)}{\le}& \max\{\Pr_{\Sim(\nu^{(i)}),\Alg}(W_i^c), \Pr_{\nu^{(i)},\Alg}(W_i^c)\}\\
	&\overset{(iii)}{=}&  \Pr_{\nu^{(i)},\Alg}(W_i^c)
	\end{eqnarray*}
	Where $(i)$ uses the identity $|a-b| \le \max\{a,b\}$ for $a,b \ge 0$, $(ii)$ uses monotonicity of probability measures, and $(iii)$ uses the fact that $\Pr_{\Sim(\nu^{(i)}),\Alg}(W_i^c) = 1 - \Pr_{\Sim(\nu^{(i)}),\Alg}(W_i) = 1 - \Pr_{\nu^{(i)},\Alg}(W_i) = \Pr_{\nu^{(i)},\Alg}(W_i^c)$, since $W_i$ is truthful. All in all, we have
	\begin{eqnarray}
	|\Pr_{\nu^{(1)},\Alg}(E) - \Pr_{\nu^{(2)},\Alg}(E)| &\le& |\Pr_{\Sim(\nu^{(1))},\Alg}(E) - \Pr_{\Sim(\nu^{(2)}),\Alg}(E)| + \sum_{i=1}^2 \Pr_{\nu^{(i)}}(W_i^c)
	\end{eqnarray}
	The bound now follows from Lemma~\ref{Pinsker}.
	\end{proof}

\subsection{Proof of Lemma~\ref{SymmetryLemma}\label{SymProof}}
	Let $\Alg$ be a (possbily non-symmetric) algorithm. We obtain the symmetric algorithm $\Alg^{\bS_n}$ by drawning a $\sigma \sim \bS_n$, and running $\Alg$ on $\sigma(\Trs)$ with decision rule $\sigma^{-1}(\widehat{S})$. Note then that a sample from arm $a$ on $\Trs$ corresponds to a sample from arm $\sigma(a)$ on $\sigma(\Trs)$. Hence, for any $\pi \in \bS(n)$,
	\begin{eqnarray*}
	&&\Pr_{\Alg^{\bS_n},\Trs}\left[(a_1,a_2,\dots,a_T,\widehat{S}) = (A_1,A_2,\dots,A_T,S)\right] \\
	&=& \frac{1}{n!} \sum_{\sigma \in \bS_n} \Pr_{\Alg,\sigma(\Trs)}\left[(a_1,a_2,\dots,a_T,\sigma^{-1}(\widehat{S})) = (\sigma(A_1),\sigma(A_2),\dots,\sigma(A_T),S)\right]\\
	&=& \frac{1}{n!} \sum_{\sigma \in \bS_n} \Pr_{\Alg,\sigma(\Trs)}\left[(a_1,a_2,\dots,a_T,\widehat{S}) = (\sigma(A_1),\sigma(A_2),\dots,\sigma(A_T),\sigma(S))\right]\\
	&=& \frac{1}{n!} \sum_{\sigma \in \bS_n} \Pr_{\Alg,\sigma\circ \pi(\Trs)}\left[(a_1,a_2,\dots,a_T,\widehat{S}) = (\sigma\circ \pi(A_1),\sigma\circ \pi(A_2),\dots,\sigma\circ \pi (A_T),\sigma \circ \pi(S))\right]\\
	&=&  \Pr_{\Alg^{\bS_n},\pi(\Trs)}\left[(a_1,a_2,\dots,a_T,\widehat{S}) = (\pi(A_1),\pi(A_2),\dots,\pi (A_T),\pi(S))\right]
	\end{eqnarray*}
	as needed. 
	We remark that this reduction to symmetric algorithms is also adopted in \cite{castro2014adaptive}, but there the reduction is applied to classes of instances which themselves are highly symmetric (e.g., all the gaps are the same). Previous works on the sampling patterns lower bounds for $\MAB$ explicitly assume that algorithms satisfy weaker conditions \cite{garivier2016explore,carpentier2016tight}, whereas our reduction to symmetric algorithms still implies bounds which hold for possibly non-symmetric algorithms as well.

%% file: multi_hypothesis_proofs.tex
\section{Proof of Theorem~\ref{Thm:BestArmSubset}\label{sec:proof_of_best_arm_subset}}
	In Theorem~\ref{Thm:BestArmSubset}, we consider the simplified case $\nu_2 = \nu_3 = \dots = \nu_n$, and fix a symmetrized algorithm $\Alg$, and the best arm has mean $\nu_1$. We will actually prove a slightly more technical version of Theorem~\ref{Thm:BestArmSubset}, from which the theorem follows as an immediate corollary.

	Recall that the intuition behind Theorem~\ref{Thm:BestArmSubset} is to show that, until sufficiently many samples has been taken, one cannot differentiate between the best arm, and other arms which exhibit large statistical deviations. To this end, we construct a simulator which is truthful as long as the top arm is not sampled too often. Fix a $\tau \in \N$ and define the simulator $\Sim$ by
	\begin{eqnarray}
	\Sim: \hat{X}_{[s,a]}\mapsfrom \begin{cases}  X_{[s,a]} & s \le \tau \\
	\overset{i.i.d.}{\sim} \nu_2 & s > \tau
	\end{cases}
	\end{eqnarray}
	Since $\Sim$ only depends on the first $\tau$ samples from $\nu$, we can use Fano's inequality to get control on events under simulated instances:
	\begin{lemma}\label{SpecializedFano} For any random, $\calF_T$-measurable subset $\calA$ of $[n]$ with $|\calA| = m$,
	\begin{eqnarray}
	\Pr_{\pi \sim \bS_n}\Pr_{\Sim(\pi(\nu)),\Alg}[\pi(1) \notin \calA] &\ge& 1 - \frac{ \tau \Delta^2 + \log 2 }{\log(n/m)}
	\end{eqnarray}
	 where $\Delta^2 := \KL(\nu_{1},\nu_2) + \KL(\nu_2,\nu_{1})$.
	\end{lemma}
	If we take $\calA = \widehat{S}$ to be the best estimate for the top arm in the above lemma, we conclude that unless $\tau \gg \Delta^{-2}\log n$, running $\Alg$ on $\Sim(\nu)$ won't be able to identify the best arm. Hence, $\Alg$ will need to break the simulator by collecting more than $\tau$ samples. More subtly, we can take $\calA$ to be the set of the first $m$ arms pulled more than $\tau = \Delta^{-2} \log (n/m)$ times (where $|\calA| < m$ if fewer than $m$ arms are pulled $\tau$ times). By Lemma~\ref{SpecializedFano}, $\mathcal{A}$ won't contain the top arm a good fraction of the time. But we know from the previous argument that the top arm is sampled at least $\tau$ times, which implies that with constant probability, there will be $m$ arms pulled at least $\tau$ times. In summary, we arrive at the following proposition which restates Theorem~\ref{Thm:BestArmSubset}, as well as proving that the top arm must be pulled $\Omega(\Delta^{-2} \log n)$ times:
	\begin{proposition}\label{Appendix:BestArmProp}
	Let $\Alg$ be $\delta$-correct, consider a game with best arm $\nu_1$ and $n-1$ arms of measure $\nu_2$. For any $\beta \ge 0$, define $S_{m,\beta} := \left\{a: N_a(T) > \Delta^{-2}\left(\beta \log\frac{n}{m} - \log 2 \right)\right\}$. Then,
	\begin{eqnarray}
	&&\label{topline} \Pr_{\pi \sim \bS_n}\Pr_{\pi(\nu),\Alg}\left[N_{\pi(1)}(T) \ge \Delta^{-2}  (\beta \log n - \log 2) \right] \ge 1 - (\beta + \delta) \quad \mathrm{and} \\
	&&\label{Appendix:secondline} \Pr_{\pi \sim \bS_n} \Pr_{\pi(\nu),\Alg}\left[ \{\pi(1) \in S_{m,\beta}\} \wedge  \{|S_{m,\beta}| \ge m\}\right] \ge 1 -  2\beta  - \delta 
	\end{eqnarray}
	\end{proposition}
	Note that Theorem~\ref{Thm:BestArmSubset} follows from Equation~\label{Appendix:secondline} by taking $\beta = 1/16$ and $\delta \le 1/8$. 

		\begin{proof}[Proof of Proposition~\ref{Appendix:BestArmProp}] Throughout the proof, will use the elementary inequality that for any events $A$ and $B$, $\Pr[A] \le \Pr[A \cap B] + \Pr[B^c]$ without comment. Let's start by proving Equation~\ref{topline}. Define $W_{\pi} = \{ N_{\pi(1)}(T) \le \tau\} $, and let $W$ to be corresponding events when $\pi$ is taken to be the identity. We see $W_{\pi}$ is $ \calF_T$-measurable, and if $\Alg^{\bS_n}$ is the symmetrized algorithm obtained from $\Alg$, then
		\begin{eqnarray}
	\Pr_{\pi \sim \bS_n}\Pr_{\pi(\nu),\Alg}[W_{\pi}] = \Pr_{\nu,\Alg^{\bS_n}}[W] 
	\end{eqnarray}
	Hence, it suffices to assume that $\Alg$ is symmetric and work with $\pi$ being the identity. Since the first $\tau$ samples from arm $1$ under $\Sim(\nu)$ are i.i.d from $\nu_1$, and since all samples from all other arms are i.i.d from $\nu_2$, we see that
	\begin{claim}\label{Btruthful} $W$ is a truthful event for $\Sim(\nu)$. 
	\end{claim}
	This implies that
	\begin{eqnarray*}
	\Pr_{\nu,\Alg}[W] &\le& \Pr_{\nu,\Alg}[W \wedge \{\hat{a} = 1\}] + \Pr_{\nu,\Alg}[\{\hat{a} \ne 1 \}]\\
	&\le& \Pr_{\nu,\Alg}[W \wedge \hat{a} = 1] + \delta\\
	&\overset{(i)}{=}& \Pr_{\Sim(\nu),\Alg}[W \wedge \hat{a} = 1] + \delta\\
	&\le& \Pr_{\Sim(\nu),\Alg}[\hat{a} = 1] + \delta ~,
	\end{eqnarray*}
	where $(i)$ follows from the following Claim~\ref{Btruthful}. Hence, Lemma~\ref{SpecializedFano} implies
	\begin{eqnarray}
	\Pr_{\Sim(\nu),\Alg}[\{\hat{a} = 1\}] \le \frac{ \Delta^{2}\tau + \log 2}{\log n}~.
	\end{eqnarray}
	For the next part, we may also assume without loss of generality that $\Alg$ is symmetric. Define the set $A_{t} = \{i: N_i((t+1)\wedge T) > \tau\}$ (these are the set of arms that have been pulled more than $\tau$ times), and let $S_m = T \wedge \sup \{t: |A_t| \le m\}$ ($S_m$ is the last time that $A_t$ is no larger than $m$). Note that $S_m$ is indeed a stopping time wrt to $\{\calF_t\}$, since the $t+1$-th arm to be sampled is determined by all the samples seen up to time $t$, and internal randomness in $\Alg$. Hence, we have that
	\begin{eqnarray*}
	\Pr_{\nu,\Alg}\left[\left\{\left|\{a:N_a(T) > \tau\}\right| \le m\right\}\right] &\le& \Pr_{\nu,\Alg}\left[W^c \cap \left\{\left|\{a:N_a(T) > \tau\}\right| \le m\right\}\right] + \Pr_{\nu,\Alg}[W]\\
	&\overset{(i)}{\le}& \Pr_{\nu,\Alg}[1 \in A_{S_m} ] + \Pr_{\nu,\Alg}[W]~,
	\end{eqnarray*}
	where $(i)$ follows because, under the event $\{\left|\{a: N_a(T) > \tau\}\right| \le m\}$, then $S_m = T$,  and thus $A_{S_m} = \{a: N_a(T) > \tau\}$. But on $W^c$, $N_1(T) > \tau$, and thus $1 \in A_{S_m}$. $ \Pr_{\nu,\Alg}[W]$ is already bounded by part $1$; for part $2$ we need the following claim to invoke a reduction: 
	\begin{claim}\label{AsmClaim}
	$\Pr_{\Sim(\nu),\Alg}[1 \in A_{S_m}] = \Pr_{\nu,\Alg}[1 \in A_{S_m}]$. Moreover, if $\Alg$ is symmetrized, then $\Pr_{\pi \sim \bS_n} \Pr_{\Sim(\pi(\nu)),\Alg}[\pi(1) \in A_{S_m}] = \Pr_{\Sim(\nu),\Alg}[1 \in A_{S_m}]$. 
	\end{claim}
	\noindent The first part of this claim holds because then event $\{1 \in A_{S_m}\}$ depends only on the first $\tau$ samples drawn from arm $1$, and the first $\tau$ samples from arm $1$ are i.i.d from $\nu_1$ under both the simulator and the true measure. The second part of the claim follows directly from the definition of symmetry, since the even $1 \in A_{S_m}$ does not depend on how the arms are labeled. Thus, invoking Lemma~\ref{SpecializedFano},
	\begin{eqnarray}
	\Pr_{\pi \sim \bS_n} \Pr_{\Sim(\pi(\nu)),\Alg}[\pi(1) \in A_{S_m}] \le \frac{ \Delta^{2}\tau + \log 2}{\log n/m}~.
	\end{eqnarray}
	Putting pieces together, we conclude that
	\begin{eqnarray}
	\Pr_{\nu,\Alg}\left[\left\{\left|\{i : N_i(T) > \tau\}\right| \le m \right\}\right] \le \delta + (\Delta^2 \tau + \log 2)(\frac{1}{\log n}+\frac{1}{\log (n/m)}) \le \delta +2 \frac{\Delta^2 \tau + \log 2}{\log (n/m)}~.
	\end{eqnarray}
	Setting $\tau = \Delta^{-2}(\beta \log(n/m) - \log 2)$ concludes. 
	\end{proof}
\subsection{Proof of Lemma~\ref{SpecializedFano}}
For $i \in \{1,\dots,n\}$, let $\nu^{(i)}$ denote the instance where $\nu_i = \nu_1$, and $\nu_j = \nu_2$ for $j \ne i$. Let $\Pr_i$ denote the law of the transcipt $\Sim(\nu^{(i)})$. We beging by applying a slight generalization of Fanos Inequality:
\begin{lemma}[Inexact Fano]\label{InexactFano} Let $X$ be a random variable, and let $E$ be a binary random variable, and suppose that $Y$ is a random variables such that $X$ and $E$ are conditionally independent given $Y$ (i.e. $X \to Y \to E$ form a Markov Chain). Then,
	\begin{eqnarray}
	P(E = 1) \ge 1 - \frac{  I(X;Y) +\log(2) }{H(X) - H(X|E = 0,\hat{X})}
	\end{eqnarray}
	where $I(X;Y)$ denotes the mutual information between $X$ and $Y$, $H(X)$ denotes the entropy of $X$, and  $H(X|E = 0)$ denotes the conditional entropy of $X$ given $E = 0$ (for details, see\cite{cover2012elements})
\end{lemma}
To apply the bound, let $\pi \sim \bS_n$, let $X = \pi(1)$, let $Y$ denote the transcript $\widehat{\Trs}$ under the distribution $\Sim(\nu^{(\pi(1))}$, and let $E = \I(\{\pi(1) \in \mathcal{A}\}) = \I(\{X \in \mathcal{A}\})$. Then $X \to Y \to E$ forms a markov chain. Since $|\mathcal{A}| = m$, on the event $E = 0$, $X$ can take at most $m$ values, namely those in $\mathcal{A}$. Hence, using a standard entropy bound \cite{cover2012elements}, $H(X|E) \le \log m$. On the other hand, since $X$ is uniform, $H(X) = \log n$, and thus $H(X) - H(X|E = 0,\hat{X}) \ge \log n/m$. 

Thus, to conclude, it suffices to show that $ I(X;Y) \le \tau \Delta^2$. Let $\bar{\Pr}$ denote the marginal of $Y$, that is, $\Pr_{X}$, where $X \overset{unif}{\sim} \{1,\dots,n\}$. Then, a standard application of Jensen's inequality (see \cite{cover2012elements} for details) gives
\begin{eqnarray}
I(X;Y) &:=&  \sum_{i=1}^n\Pr(X = i)\KL(\Pr_i,\bar{\Pr}) \le \sum_{j,i=1}^M\Pr(X = j)\Pr(X = i)\KL(\Pr_i,\Pr_j)
\end{eqnarray}
For $i = j$, $\KL(\Pr_i,\Pr_i) = 0$. For $i \ne j$, we use the independence of the entries of the transcript to compute
\begin{eqnarray*}
\KL(\Pr_i,\Pr_j) &=& \sum_{a=1}^n \sum_{s=1}^{\infty} \KL(\widehat{X}_{[a,s]} \big{|} \Sim(\nu^{(i)}), \widehat{X}_{[a,s]} \big{|} \Sim(\nu^{(j)})\\
&\overset{(i)}{=}&  \sum_{s=1}^{\tau} \KL(\widehat{X}_{[i,s]} \big{|} \Sim(\nu^{(i)}), \widehat{X}_{[i,s]} \big{|} \Sim(\nu^{(j)}) + \KL(\widehat{X}_{[j,s]} \big{|} \Sim(\nu^{(i)}), \widehat{X}_{[j,s]} \big{|} \Sim(\nu^{(j)}) \\
&\overset{(ii)}{=}&  \sum_{s=1}^{\tau} \KL(\nu_1,\nu_2) + \KL(\nu_2,\nu_1) = \tau \Delta^2~,
\end{eqnarray*}
where $(i)$ follows since the law of $\widehat{X}_{[a,s]}$ differs between $\Sim(\nu^{(i)})$ and $\Sim(\nu^{(j)})$ for $a \in \{i,j\}$ and $s \in \{1,\dots,\tau\}$, and $(ii)$ follows from the construction of our simulator. Hence, 
\begin{eqnarray}
I(X;Y) &\le&  \sum_{j,i=1}^M\Pr(X = j)\Pr(X = i)\KL(\Pr_i,\Pr_j) = \sum_{j,i=1}^M \frac{\tau\Delta^2}{n^2}\I(i \ne j) \le \tau\Delta^2
\end{eqnarray}

%% file: tiltings_explanation.tex
\subsection{High Level-Intuition For Proposition~\ref{GaussianProp}\label{TiltingsExplanation}}

	As in the other results in this paper, the key step boils down to designing an effective simulator $\Sim$. Unlike the prior bounds, we need to take a lot of care to quantify how $\Sim$ limits information between instances. 

	To make things concrete, suppose that the base instance is $\nu$ with best arm index $1$, and where the measures $\nu_i$ are Gaussians with means $\mu_i$ and variance $1$. For clarity, suppose that the gaps are on the same order, say $\Delta \le \mu_1 - \mu_b \le 2\Delta$ for all $b \ge 2$. Since our goal is to show that the best arm must be pulled $\gtrsim \Delta^{-2} \log n$ times on average, a natural choice of a truthful event is $W = \{N_1(T) \le \tau\}$  for some $\tau \gtrsim \Delta^{-2} \log n$. This suggests that our simulator should always return the true samples $X_{[a,s]}$ from $\Trs$ for all arms $a \ne 1$, and the first $\tau$ samples from arm $1$. 

	Once $\tau$ samples are taken from arm $1$, our $\Sim$ will look at the first $\tau$ samples from each arm $j \ne 1$, and pick an index $\widehat{j}$ such that the first $\tau$ samples $X_{[\widehat{j},1]},\dots,X_{[\widehat{j},\tau]}$ ``look like'' they were drawn from the distribution $\nu_1$. We do this by defining events $E_j$ which depend on the first $\tau$-samples from arm $j$, as well as some internal random bits $\xi_j$, and choosing $\widehat{j}$ uniformly from the arms $j$ for which $E_j$ holds. In other word, our simulator is given by 
	\begin{eqnarray}
	&\Sim(\nu): \widehat{X}_{[a,s]} \mapsfrom \begin{cases} X_{[a,s]} & a \ne 1 \\
	X_{[1,s]} & a = 1, s \le \tau \\
	\overset{i.i.d}{\sim} \nu_{\widehat{j}} & a = 1, s > \tau
	\end{cases}\quad \text{where} \\
	\ &\widehat{j} = \begin{cases} \overset{unif}{\sim} \{j \ne 1: E_j \text{ holds}\} & \text{if at least one } E_j \text{ holds}\\
	1 & \text{otherwise}
	\end{cases}
	\end{eqnarray}
	Our construction will ensure that at least one $E_j$ will hold with constant probability. Hence, the only information which can distinguish between the arms $1$ and $\widehat{j} \ne 1$ are the first $\tau$ samples from each arm. But if the first $\tau$ samples from arm $\widehat{j}$ ``look'' as if they were drawn from $\nu_1$, then this information will be insufficient to tell the arms apart. In other words, we can think of $\Sim$ as forcing the learner to conduct an adversarially-chosen, \emph{data-dependent} two-hypothesis test: is the best arm $1$ or arm $\widehat{j}$ ?

	What's left is to understand why we should even expect to find an arm $\widehat{j}$ whose first $\tau = O(\Delta^{-2} \log n)$ samples resemble those from arm $1$. The intuition for this is perhaps best understood in terms of Gaussian large-deviations. Indeed, consider the empirical means of each arm $\widehat{\mu}_{j,\tau} = \frac{1}{\tau}\sum_{s = 1}^{\tau} X_{[j,s]}$. Then for any \emph{fixed} $j \in [n]$, we have that $|\mu_j - \widehat{\mu}_j| \lesssim \sqrt{1/\tau}$. However, Gaussian large deviations imply that for \emph{some} arm $\widehat{j} \in \{2,\dots,n\}$, the empirical mean will overestimate its true mean by a factor of $\approx \sqrt{\log(n)/\tau}$ (that is $\widehat{\mu}_{\widehat{j},\tau} \ge \mu_{\widehat{j}} + \Omega(\sqrt{\log(n/\delta)/\tau})$). By the assumption that $\Delta \le \mu_1 - \mu_{\widehat{j}} \le 2\Delta$, the large deviation combined with a confidence interval around arm $1$ implies that unless $\tau \gtrsim \Delta^{-2} \log n$, there will be an arm $\widehat{j}$ whose empirical mean is larger the empirical mean of arm $1$; thereby ``looking'' like the best arm.

	Unfortunately, this intuition is not quite enough for a proof. Indeed, if $\tau \ll \Delta^{-2} \log n$, then with good probability the the arm with the greatest empirical mean  will not be best arm. This leads to a paradox: suppose $\tau \ll \Delta^{-2} \log n$, and the learner is given a choice between two arms - one of which has the highest empirical mean, and one of which is assured to be the best arm. Then the learner should guess that the best arm is the one with the \emph{lesser} of the two empirical means!
\subsection{Tiltings}
	To get around this issue, we pick $\widehat{j}$ using a technique called ``tilting'', which is the key technical innovation behind this result. Given $\tau$ samples from arm $j$, and access to some random bits $\xi_j$, the goal is to construct an event $E_j$ (depending on the $\tau$ samples from arm $j$, as well as $\xi_j$) such that conditioning on $E_j$ ``tilts'' the distribution of the first $\tau$ samples from an arm $j$ to ``look like'' samples from arm $1$. Since the sample mean is a sufficient statistic for Gaussians, it is sufficient to ensure that the distribution of the sample means $\widehat{\mu}_{j,\tau}$ are close in distribution. The basic idea is captured in the following proposition:
	\begin{proposition}[Informal]\label{InformalProp}
	Let $\xi_j \sim \mathrm{Uniform}[0,1]$ and independent of everything else, and let $p \in (0,1)$. If  $\tau \ll \Delta^{-2} \log (1/p)$, then there exists a deterministic function $\calK_j: \R \to [0,1]$ such that the following holds: Define the event $E_{j} = \{\calK_j(\widehat{\mu}_{j,\tau}) \le \xi_j \}$. Then, the conditional distribution of $\widehat{\mu}_{j,\tau}$ on $E_j$  ``looks like'' the distribution of $\widehat{\mu}_{1,\tau}$, in the sense that the $\TV(\widehat{\mu}_{1,\tau}; \widehat{\mu}_{j,\tau} \big{|} E_j) = o(1)$. Moreover, $E_j$ holds with probability at least $p$.
	\end{proposition}


	\begin{figure}[t!]
\centering
\begin{subfigure}[b]{.45\textwidth}
   \includegraphics[width=\textwidth]{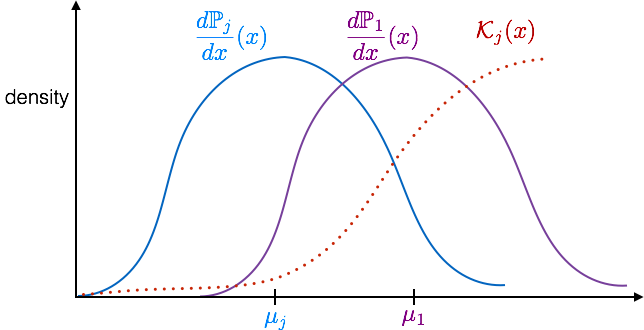}
  \caption{Before Tilting}
  \label{fig:sub1}
\end{subfigure}%
\begin{subfigure}[b]{.45\textwidth}
  \includegraphics[width=\textwidth]{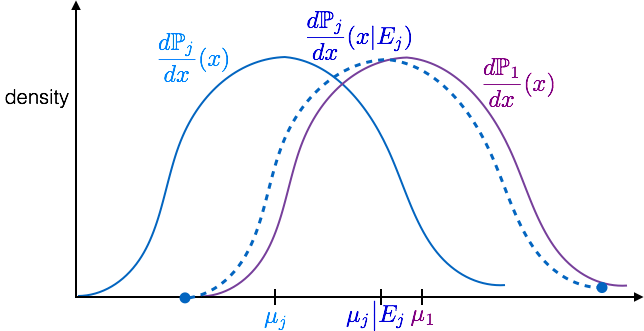}
  \caption{After Tilting}
  \label{fig:sub2}
\end{subfigure}
\caption{ The event $E_j$ depends on the samples from arm $j$. Thus, conditioning on $E_j$ ``tilts'' the distribution of the those samples.}
\label{fig:reweight}
\end{figure}

	Since $\xi_j$ is uniform, $\calK_j(\widehat{\mu}_{j,\tau}) = \Pr_{\xi_j}(E_j\big{|} \widehat{\mu}_{j,\tau}) $. Thus, up to normalization, conditioning on the event $E_j$ reweights the density of $\widehat{\mu}_{j,\tau}$ by the value of $\calK_j(\widehat{\mu}_{j,\tau})$, thereby tilting its shape to resemble the distribution of $\widehat{\mu}_{1,\tau}$. This is depicted in Figure~\ref{fig:reweight}. The random numbers $\xi_j$ are essential to this construction, since they let us reweight the distribution of $\widehat{\mu}_{j,\tau}$ by fractional values. Since $\calK_j$ is bounded above by one, reweighting doesn't come for free, and our major technical challenge is to choose $\calK_j$ so as to ensure that $\Pr(E_j) = \Exp[\calK_j(\widehat{\mu}_{j,\tau})]$ is at least $p$. This sort of construction is known in the probability literature as ``tilting'', and is used in the Herbst argument in the concentration-of-measure literature (Chapter 3 of \cite{raginsky2014concentration})\footnote{Unlike our construction, the Herbst argument tilting reweights by an unbounded function $\calK_j$ (rather than a function bounded in $[0,1]$), and thus those tiltings \emph{cannot} be interpreted as a conditioning on an event}.  To the best our knowledge, this constitutes the first use of tiltings for proving information theoretic lower bounds.

	To conclude our simulator argument, we apply Proposition~\ref{InformalProp} with $p = (10/n)$ and $\tau \approx \Delta^{-2} \log\frac{n-1}{10} \approx \Delta^{-2} \log n$. Then for any fixed arm $j$, $E_j$ will hold with  probability at least say $10/(n-1)$ (say $n \gg 10$), on which the first $\tau$ samples from arm $j$ will ``look-like'' samples from $\nu_1$, in $\TV$ distance. Hence, with probability $1 - (1 - 10/(n-1))^{n-1} \ge 1 - e^{-10} \ge .999$, there will exists an arm $\widehat{j}$ such that $E_{\widehat{j}}$ holds, and thus the first $\tau$ samples from $\widehat{j}$ ``look-like'' samples from $\nu_1$, in $\TV$. In particular, if our simulator chooses $\widehat{j}$ uniformly from the arms $j$ such that $E_j$ holds (and takes $\widehat{j} = 1$ otherwise), then with probability $.999$, our simulator can confuse the learner by showing her two arms the distribution of whose samples look like $\nu_1$, as needed.

 	\subsubsection{Data-Dependent Two-Hypothesis Testing}

 	Recall above that $\Sim$  forces the learner to perform a data-dependent two hypothesis test  - ``is the best arm $1$ or $\widehat{j}$ '' - chosen adversarially from the set of two-hypothesis tests ``is the best arm $1$ or $j$'' for $j \in \{2,\dots,n\}$. We emphasize that the argument from Proposition~\ref{InformalProp} is very different than the familiar reductions to $n$-way or composite hypothesis testing problems. Observe that
	\begin{enumerate}
		\item By giving the learner the choice between only arms $1$ and $\widehat{j}$, the adversarial two-hypothesis test reduces the learner's number of possible  hypotheses for the best arm from $n$ down to $2$. Thus, this problem is potentially easier than the $n$-way hypothesis test corresponding to best-arm identification. In particular, Proposition~\ref{InformalProp} is \emph{not implied by Fano's Equality or other $n$-way testing lower bounds} 
	\item By the same token, the adversarial two-hypothesis test is also potentially easier than the \emph{composite} hypothesis test: is $1$ the best arm, or is another arm $j \in \{2,\dots,n\}$ the best arm? Hence, Proposition~\ref{InformalProp} is \emph{not implied by lower bounds on composite hypothesis tests}.
		\item On the other hand, since $\widehat{j}$ depends on the observed data in this adversarial way, the adversarial two-way hypothesis is strictly \emph{harder} than the standard oblivious two-hypothesis test which \emph{fixes} $j$ \emph{in advance} and asks: is the best arm $1$ and some $j$? Indeed, fixing the two-hypothesis test in advance does not force the learner to incur a log-factor in the sample complexity.
	\end{enumerate}

%% file: distinct_measures.tex


\subsection{Statement of the Main Technical Theorem\label{sec:main_tec_theorem}}

	Our main theorem is stated for single parameter exponential families\citep{nielsen2009statistical}, which we define for the sake of completeness in Section~\ref{CensoredSection}.
	\begin{theorem}\label{Big Main Theorem}  Let $\nu$ be a measure with best arm such that each $\nu_j$ comes from an exponential family $\{p_{\theta}\}_{\theta \in \Theta}$ with corresponding parameter by $\theta_j \in \Theta$, and that $[\theta_j,2\theta_1 - \theta_j] \subset \Theta$.  Suppose that $\Alg$ is $\delta$-correct, in the sense that for any $\pi \in \bS_n$, $\Alg$ can identify the unique arm of $\pi(\nu)$ with density $\nu_1 = p_{\theta_1}$ among with probability of error at most $\delta$. Then, for all $\alpha > 0$
	\begin{equation}
	\begin{aligned}
	&\Exp_{\pi \sim \bS_n}\Pr_{\pi(\nu),\Alg}[\{N_{\pi(1)}(T) > \DeltaEff^{-2}\log(n/\alpha) \}] \ge \sup_{\kappa \in [0,1]}\frac{1}{2} \left(1 - e^{-\alpha\kappa(1-\kappa)}\right) \left(1 - 2\kappa \right) - \delta \\
	&\mathrm{where} \quad \DeltaEff^2 = \max_{j > 1} \kl(\theta_1,\theta_j) + \kl(2\theta_1  - \theta_j,\theta_j )
	\end{aligned}
	\end{equation}
	\end{theorem}
	Furthermore, observe that for Gaussian rewards with unit variance, $\DeltaEff^2$ corresponds exactly with the largest squared gap $(\theta_1 - \theta_j)^2$. By considering best-arm subproblems with the top $m \le n$ arms, we arrive at the following corollary, which immediate specializes to Proposition~\ref{GaussianProp} with Gaussian rewards: 
	\begin{corollary}\label{CorollaryToBigMain}
	In setting of Theorem~\ref{Big Main Theorem}, we have the following lower bound for every $m \le n$ and $\alpha > 0$,
	\begin{equation}
	\begin{aligned}
	&\Exp_{\pi \sim \bS_n}\Pr_{\pi(\nu),\Alg}[\{N_{\pi(1)}(T) > \DeltaEff(m)^{-2}\log(m/\alpha) \}] \ge \sup_{\kappa \in [0,1]}\frac{1}{2} \left(1 - e^{-\alpha\kappa(1-\kappa)}\right) \left(1 - 2\kappa \right) - \delta \\
	&\mathrm{where} \quad \DeltaEff(m)^2 \text{ is the } (m-1)\text{-th smallest value of } \{\kl(\theta_1,\theta_j) + \kl(2\theta_1  - \theta_j,\theta_j )\}_{j} 
	\end{aligned}
	\end{equation}
	\end{corollary}
\subsection{Censored Tilting\label{CensoredSection}}
	In this section, we are going to formally construct the events $E_j$. We will first illustrate the idea for a a generic collection of random variables, and then show how to specialize for bandits. For each $j \in [n]$, we consider a Markov Chain $Z_j \to E_j = 1$, where $Z_j$ is a real valued random variable, and $E_j$ is an event depending only on $Z_j$. Under suitable technical conditions, the distribution $(Z_j,\I(E_j))$ is then defined by a Markov Kernel $\calK_j: \R \to [0,1]$, where $\Pr(E_j | Z_j = z) = \calK_j(z)$. Conversely, any such Markov Kernel induces a joint distribution on $(Z_j,\I(E_j))$. To replicate the malicious adversary from Proposition~\ref{InformalProp}, we can represent $E_j$ explictly by letting $\xi_j \sim \mathrm{Uniform}[0,1]$ and independent of everything else, and setting $E_j = \{\xi_j \le \calK_j(z)\}$. 

	We will say that $\calK_j$ is \emph{nondegenerate} if $\Pr(E_j = 1) \equiv \Exp[\calK_j(Z_j)] > 0$. When $Z_j$ has a density to a measure $\eta(x)$, and $\calK_j$ is nongenerate, then Baye's rule implies
	\begin{eqnarray}\label{EmpiricalTilting}
	\frac{d\Pr_{Z_j}}{d\eta}(x| E_j) = \dfrac{ \calK_j(x)}{\Exp[\calK_j(Z_j)]}\cdot \frac{d\Pr_{Z_j}}{d\eta}(x) 
	\end{eqnarray}
	In other words, conditioning on the event $E_j$ ``tilts'' the density of $Z_j$ by a function $\calK_j(x) / \Exp[\calK_j(Z_j)] $. We will call tiltings that arise in this fashion a \emph{censored tilting}. Indeed, imagine an observer who tries to measure $Z_j$. On $E_j$, she gets a proper measurement of $Z_j$, but on $E_j^c$ she is censored. Then the censored tilting $\Pr(Z_j | E_j)$ describes the distribution of the observers non-censored measurements. Keeping with this metaphor, we  will call $E_j$ the \emph{measuring event} induced by $\calK_j$. 
	\begin{remark}
	Tiltings appear as a step in the Herbst Argument for proving concentration of measure bounds from Log-Sobolev inequalitys. In that setting, one tilts by potentially unbounded functions $g_j \ge 0$ that need only satisfy the integrability condition $\Exp[g_j(Z_j)] < \infty$. In our setting, this tilting to arise from a function $\calK_j \in [0,1]$, since $\calK_j$ corresponds to a conditional probability operator. 
	\end{remark}
	To apply this idea to $\MAB$, fix a measure $\nu$ with decreasing means $\mu_1 > \mu_2 \ge \dots \mu_n$. Given a transcript $\Trs$ and $\tau \in \N$, let $\overline{X}_{j,\tau} = \frac{1}{\tau}\sum_{s = 1}^{\tau} X_{[j,s]}$. We will simply write $\overline{X}_j$ when $\tau$ is clear from context. To simplify things, we shall assume that all the measures $\nu_j$ come from a cannonical exponential family of densities $p_{\theta}(x) = \exp(\theta x - A(\theta))d\eta(x)$ with respect to a measure $\eta(x)$ where $\theta$ lie in a convex subset $\Theta$ of $\R$. It is well known that this implies that 
	\begin{lemma}[\cite{nielsen2009statistical}]
	Suppose that $\nu_j$ has density $p_{\theta_j}(x) = \exp(\theta_j x - A(\theta_j))$ with respect to a measure $\eta$. Then,
	\begin{enumerate}
	\item $\overline{X}_{j,\tau}$ is a sufficient statistic for $X_{[j,1]},\dots,X_{[j,\tau]}$
	\item There exists a measure $\eta_{\tau}(x)$ on $\R$, such that $\overline{X}_{j,\tau}$ has density $q_{\tau \theta_j} (x) := \exp( \tau \theta_j x - \tau A(\theta_j))d\eta_{\tau}(x)$ with respect $\eta_{\tau}(x)$.
		\end{enumerate}
		In particular, the densities $q_{\tau \theta}(x)$ for $\theta \in \Theta$ form an exponential family.

	\end{lemma}

	Now, for each $j$, consider tiltings of the form $\calK_j(x) = \frac{e^{\tau (\theta_1 - \theta_j) x}}{c_j} \I( e^{\tau (\theta_1 - \theta_j) x} \le c_j)$. Then,
	\begin{eqnarray}\label{TiltedDensity}
	\frac{d\Pr_{Z_j}}{d\eta}(x| E_j) \propto e^{\tau \theta_j x}\cdot e^{\tau (\theta_1 - \theta_j) x)} \cdot \I( e^{\tau (\theta_1 - \theta_j) x} \le c_j) = e^{\tau \theta_1 x} \I( e^{\tau (\theta_1 - \theta_j) x} \le c_j)
	\end{eqnarray}
	Since $\frac{d\Pr_{Z_j}}{d\eta}(x| E_j)$ is a density, the uniquess of normalization implies the following facts:
	\begin{lemma}\label{Censoring} Let $\calK_j(x) = \frac{e^{\tau (\theta_1 - \theta_j) x}}{c_j} \I( e^{\tau \theta_j x} \le c_j)$, and let $E_j$ be the corresponding measuring event. Then,
	\begin{enumerate} 
	\item The censored tilting of $\overline{X}_j | E_j$ has the distribution of $\overline{X}_1 \big{|}\{ e^{\tau(\theta_1 - \theta_j) \overline{X}_j} \le c \}$
	\item $\TV( \Pr_{\overline{X}_1}, \Pr_{\overline{X}_j}[\cdot \big{|} E_j]) = \Pr(e^{\tau(\theta_1 - \theta_j) \overline{X}_1} > c) $
	\item $\Pr(E_j) = \frac{1}{c}(1- Q_j(E_j))\cdot e^{\tau  \{A (\theta_1) - A(\theta_j)\}} $
	\end{enumerate}
	\end{lemma}
	\begin{proof}
	The first point follows from Equation~\ref{TiltedDensity}. The second point follows directly from Lemma~\ref{ConditionalTV}, and the last point follows from the following computation:
	\begin{eqnarray*}
	\Pr(E_j) &=& \frac{1}{c}\cdot\Exp\left[\exp(\tau(\theta_1 - \theta_j) \overline{X}_j) \cdot \I(e^{\tau(\theta_1 - \theta_j)\overline{X}_j} \le c) \right] \\
	&=& \frac{1}{c}\cdot \int \exp(\tau(\theta_1 - \theta_j) x)\exp(\tau\theta_jx - \tau A(\theta_j))\I(e^{\tau(\theta_1 - \theta_j)x} \le c) d\eta_{\tau}(x)\\
	&=& \frac{1}{c}\cdot \int \exp(\tau \theta_1x - \tau  A (\theta_j) ))\I(e^{\tau(\theta_1 - \theta_j)x} \le c) d\eta_{\tau}(x)\\
	&=& \frac{1}{c}\cdot \int \exp(\tau \theta_1x - \tau  A (\theta_1) + \tau  \{A (\theta_1) - A(\theta_j)\} ))\I(e^{\tau(\theta_1 - \theta_j)x} \le c) d\eta_{\tau}(x)\\
	&=& \frac{e^{\tau  (A (\theta_1) - A(\theta_j)}}{c}\cdot \int \exp(\tau \theta_1 x - \tau  A (\theta_1) )\I(e^{\tau(\theta_1 - \theta_j)x} \le c) d\eta_{\tau}(x)\\
	&=& \frac{e^{\tau  \{A (\theta_1) - A(\theta_j)\} }}{c}\Pr[e^{\tau(\theta_1 - \theta_j)\overline{X}_1} \le c] 
	\end{eqnarray*}
	\end{proof}
	The last point follows from the following lemma, proved in Section~\ref{ConditionalTVProof}:
	\begin{lemma}[$\TV$ under conditioning]\label{ConditionalTV} Let $\Pr$ be a probability measure on a space $(\Omega,\calF)$, and let $B \in \calF$ have $\Pr(B) > 0$. Then,
	\begin{eqnarray} 
	\TV( \Pr[\cdot], \Pr[\cdot | B]) = \Pr[B^c] 
	\end{eqnarray}. 
	\end{lemma}
\subsection{Building the Simulator}
	\subsubsection{Defining the Simulator on $\nu$}
	Again, let $\nu$ be an instance with best arm $\nu_1$, and fix a $\tau \in \N$. Our simulator will always return the true samples $X_{[a,s]}$ from $\Trs$ for all arms $a \ne 1$, and for the first $\tau$ samples from arm $1$. After $\tau$ samples are taken from arm $1$, the samples will be drawn independently from the measure $\nu_{\hat{j}}$, where $\hat{j} \in [n]$ is a maliciously chosen index which we will define shortly, using the events $E_j$ in the previous section. To summarize,
	\begin{eqnarray}
	\Sim(\nu): \widehat{X}_{[a,s]} \mapsfrom \begin{cases} X_{[a,s]} & a \ne 1 \\
	X_{[1,s]} & a = 1, s \le \tau \\
	\overset{i.i.d}{\sim} \nu_{\widehat{j}} & a = 1, s > \tau
	\end{cases}
	\end{eqnarray}

	\begin{fact}\label{BtruthfulFact} If $W = \{N_1(T) \le \tau\}$, then $\Alg$ is truthful on $W$ under $\Sim(\nu)$.
	\end{fact}
	\begin{proof}
	The only samples which are altered by $\Sim(\nu)$ are those taken from arm $1$ after arm $1$ has been sampled $> \tau$ times.
	\end{proof}
	Next, let's define $\widehat{j}$. Let $\overline{X}_j = \frac{1}{\tau}\sum_{s=1}^{\tau} X_{[j,s]}$, fix constants $c_2,\dots,c_n \in \R_{> 0}$ to be chosen later, and let $\calK_j$ be the corresponding Markov Kernel from Lemma~\ref{Censoring}. For each $j \in \{2,\dots,n\}$, $\Sim$ draws a i.i.d random number $\xi_j \overset{unif}{\sim} [0,1]$. The following fact is just a restatement of this definition in the language of Section~\ref{CensoredSection}:
	\begin{fact} Let $E_j = \{\xi_j \le 
	\calK_j(\overline{X}_j)\}$. Then $E_j$ is the measuring event corresponding to the Markov Kernel $\calK_j$, and are mutually independent.
	\end{fact}
	We now define the index $\widehat{j}$ and corresponding ``malicious events'' $M_j$ by
	\begin{eqnarray}
	M_j := \{\widehat{j} = j\} & \text{where} & \widehat{j} = \begin{cases} \overset{unif}{\sim} \{j: E_j \text{ occurs}\} & \text{on } \bigcup_j E_j \\
	1 & \text{ otherwise}
	\end{cases}  
	\end{eqnarray}

	\subsubsection{Defining $\Sim$ on Alternate Measures} 
	The next step is to construct our alternative hypotheses. Let $\pi_{(\ell)}$ denote the permutation which swaps arms $1$ and $\ell$, and define the measures $\nu^{(2)},\dots,\nu^{(n)}$, where $\nu^{(\ell)} = \pi_{(\ell)}(\nu)$ (note that $\pi_{(\ell)} = \pi_{(\ell)}^{-1}$). To define the simulator on these instances, we still let $\xi_j \overset{unif}{\sim} [0,1]$, and now define, for $j \in \{2,\dots,n\}$
	\begin{eqnarray}
	E_j^{(\ell)} &:=& \{\xi_{j} \le \calK_j(\overline{X}_{\pi_{(\ell)}(j)})\} \\
	\widehat{j}_{\ell}  &\overset{unif}{\sim}& \{j: E_j^{(\ell)} \text{ holds} \} \\ 
	M_j^{(\ell)} &:=& \{\widehat{j}_{\ell} = j\}
	\end{eqnarray}
	and set
	\begin{eqnarray}
	\Sim(\nu^{(\ell)}): X_{[a,s]} \mapsto \begin{cases} X_{[a,s]} & a \ne \ell \\
	X_{[a,s]} & a = \ell, s \le \tau \\
	\overset{i.i.d}{\sim} \nu_{\widehat{j}_{\ell}} & a = \ell, s > \tau
	\end{cases}
	\end{eqnarray}
	Note that this esnsure that if $\widehat{\Trs}$ is a transcript from $\Sim(\nu^{(\ell)})$, and $\widehat{\Trs}^{(\ell)}$, then $\pi_{(\ell)}^{-1}(\widehat{\Trs}^{(\ell)})$ (that is, the transcript obtained by swapping indices $1$ and $\ell$ in $\widehat{\Trs}^{(\ell)}$) has the same distribution as $\widehat{\Trs}$.


	Our construction is symmetric in the following sense:
	\begin{fact}\label{WjFact} $\Alg$ is truthful on $W_{j}:= \{N_{j)}(T) \le \tau\}$ under $\Sim(\nu^{(j)})$. Moreover, for each $j \in \{2,\dots,n\}$, $\Pr_{\Sim(\nu)}[W | M_j] = \Pr_{\Sim(\nu^{(j)})}[W_j | M_j^{(j)}]$ and $\Pr_{\Sim(\nu)}[\{\hat{y} \ne 1\} | M_j] = \Pr_{\Sim(\nu^{(j)})}[\{\hat{y} \ne j\} | M_j^{(j)}]$
	\end{fact}
	\begin{proof}
	The first point just follows since $\Sim(\nu^{(j)})$ only changes samples once arm $j$ has been pulled more than $\tau$ times. The second point follows since, the event $M_j^{(j)}$ (resp $\{\hat{y} \ne j\}$) and $W_j$ correspond to the events $M_j$ (resp $\{\hat{y} \ne 1\}$) and $W$ if the labels of arms $1$ and $j$ are swapped. But if we swap the labels of $1$ and $j$, distribution of $\widehat{\Trs}^{(j)}$ under $\Sim(\nu^{(j)})$ is identical to the distribution of $\widehat{\Trs}$ under $\Sim(\nu)$. 
	\end{proof}
	Using this symmetry, the total variation between the transcripts returned by $\Sim(\nu) $ given $E_j$ and $\Sim(\nu^{(j)}) $ given $E_j$ can be bounded as follows
	\begin{fact}\label{TVfact} Let $\overline{X}_{\ell}$ denote a sample with the distribution of $\sum_{s=1}^{\tau} X_{\ell,s}$, where each $X_{\ell,s} \sim \nu_{\ell}$. For $j \in \{2,\dots,n\}$, $\TV\left[ \Sim(\nu) \big{|} M_j; \Sim(\nu^{(j)}) \big{|} M_j^{(j)}\right] \le 2\TV(\overline{X}_j \big{|} E_j, \overline{X}_1) $.
	\end{fact}
	This fact takes a bit of care to verify, and so we defer its proof to Section~\ref{TVfactsection}.

	\subsection{Coupling together $\nu$ and $\{\nu^{(j)}\}$}
	Facts~\ref{BtruthfulFact} and~\ref{WjFact}, we can couple together the measures using a conditional analogue of the the Simulator Le Cam (Proposition~\ref{Prop:SimLeCam}), proved in Section~\ref{Conditional-LeCam-Proof}.
	\begin{lemma}[Conditional Le Cam's]\label{ConditionalLeCam} Suppose that any events $W$, $\{W_j\}$ and $M_j$ satisfy the conclusions of Facts~\ref{BtruthfulFact} and~\ref{WjFact}. Then, if $\Alg$ is symmetric, then for all $j \in \{2,\dots,n\}$
	\begin{eqnarray}\label{ConditionalEq}
	2\Pr_{\nu,\Alg}\left[W^c | M_j\right] \ge 1 - 2\Pr_{\nu,\Alg}\left[\{\hat{y} \ne 1\}\big{|} M_j\right] - \TV\left[\Sim(\nu\big{|} M_j)-\Sim(\nu^{(j)}\big{|}  M_j)\right]
	\end{eqnarray}
	\end{lemma}
	Effectively, the above lemma paritions the space into malicious events $M_j$, and applies Proposition~\ref{Prop:SimLeCam} on each part of the partition. 

	Since the events $M_j$ are disjoint, multipling the left and right hand side of Equation~\ref{ConditionalEq} by $\Pr_{\nu,\Alg}\left[M_j\right]$, setting $\overline{M} := \bigcup_{j=2}^M$ and summing yields
	\begin{eqnarray*}
	 2\Pr_{\nu,\Alg}\left[W^c \wedge \overline{M} \right]  &\ge& \Pr_{\nu,\Alg}\left[ \overline{M}\right] -  2\Pr_{\nu,\Alg}\left[\{\hat{y} \ne 1\} \wedge \overline{M} \right] \\ 
	&-& \sum_j \Pr_{\nu,\Alg}\left[W^c \wedge M_j \right]\TV\left[\Sim(\nu\big{|} M_j)-\Sim(\nu^{(j)}\big{|}  M_j)\right]
	\end{eqnarray*}
	We can bound $\Pr_{\nu,\Alg}\left[W^c \wedge \overline{M} \right] \le \Pr_{\nu,\Alg}\left[W^c \right]$ and, if $\Alg$ is $\delta$-correct, then  $\Pr_{\nu,\Alg}\left[\{\hat{y} \ne 1\} \wedge \overline{M} \right] \le \Pr_{\nu,\Alg}\left[\{\hat{y} \ne 1\}\right] \le \delta$. Finally, Holder's Inequality, the disjointness of $M_j$ and Fact~\ref{TVfact} imply
	\begin{eqnarray*}
	&&\sum_j \Pr_{\nu,\Alg}\left[W^c \wedge M_h \right]\TV\left[\Sim(\nu\big{|} M_j)-\Sim(\nu^{(j)}\big{|}  M_j)\right] \\
	&\le& \left(\sum_j\Pr_{\nu,\Alg}[M_j]\right) \cdot \max_{j}\TV\left[\Sim(\nu\big{|} M_j)-\Sim(\nu^{(j)}\big{|}  M_j)\right] \\
	&=& \Pr_{\nu,\Alg}[\overline{M}] \cdot  \max_{j}\TV\left[\Sim(\nu\big{|} M_j)-\Sim(\nu^{(j)}\big{|}  M_j)\right] \\
	&\le& \Pr_{\nu,\Alg}[\overline{M}] \cdot  2\max_{j}\TV(\overline{X}_j \big{|} E_j, \overline{X}_1) 
	\end{eqnarray*}
	where the last step is a consequence of Fact~\ref{TVfact}. Combining these bounds, and noting that $\overline{M} = \bigcup_{j=2}^n M_j \equiv \bigcup_{j=2}^n E_j$ and $W = \{N_1(T) > \tau\}$ implies the following proposition:
	\begin{proposition}\label{BreakingLowerBound} Suppose that $\Alg$ is $\delta$-correct and symmetric. Then
	\begin{eqnarray}
	2\Pr_{\nu,\Alg}[\{N_1(T) > \tau\}] \ge \Pr_{\nu}[\bigcup_{j=2}^n E_j](1 - 2\max_j \TV(\overline{X}_j \big{|} E_j, \overline{X}_1) ) - 2\delta
	\end{eqnarray}
	where we note that the probability of $\bigcup_{j=2}^n E_j$ does not depend on $\Alg$.
	\end{proposition}
	Our goal is now clear: choose the Kernel's $\calK_j$ so as to balance the terms $Pr_{\nu}[\bigcup_{j=2}^n E_j]$ and $\max_j Q_j(E_j)$ in Equation~\ref{BreakingLowerBound}.
\subsection{Proving Theorem~\ref{Big Main Theorem}}
	To conclude Theorem~\ref{Big Main Theorem}, we first introduce the following technical lemma.
	\begin{lemma}\label{Lem:Balance}
	Suppose that $\nu_j$ comes from an exponential family $\{p_{\theta}\}_{\theta \in \Theta}$ with corresponding parameter by $\theta_j \in \Theta$. If $[\theta_j,2\theta_1 - \theta_j] \subset \Theta$, then for any $\kappa > 0$, there exists a choice of $c_j$ for which the corresponding kernel $\calK_j$ has
	\begin{eqnarray}
	\TV(\overline{X}_j \big{|} E_j , \overline{X}_1 ) \le \kappa & \text{and} & \Pr(E_j) \ge \kappa(1-\kappa)e^{-\tau \{\kl(\theta_1,\theta_j) + \kl(2\theta_1  - \theta_j,\theta_j )\}}
	\end{eqnarray}
	where $\kl(\theta,\widetilde{\theta})$ denotes the $\KL$ divergence between the laws $\Pr_{\theta}$ and $\Pr_{\widetilde{\theta}}$.  
	\end{lemma}
	With this Lemma in hand, we see that taking  $\kappa > 0$, and $\tau = \log(n/\alpha) (\max_j \kl(\theta_1,\theta_j) + \kl(2\theta_1  - \theta_j,\theta_j ))^{-1}$ implies that 
	\begin{eqnarray*}
	2\Pr_{\nu,\Alg}[\{N_1(T) > \tau\}] &\ge& \left(1 - (1-\kappa(1-\kappa)\alpha/n\right)^n \left(1 - 2\kappa \right) - 2\delta\\
	 &\ge& (1 - \left(1-\frac{\alpha\kappa(1-\kappa)}{n}\right)^n )(1 - 2\kappa ) - 2\delta\\
	 &\ge& (1 - e^{-\alpha\kappa(1-\kappa)}) (1 - 2\kappa ) - 2\delta
	\end{eqnarray*}
	Moving from symmetrized algorithms to expecations over $\pi \sim \bS_n$ (Lemma~\ref{SymmetryLemma}) concludes the proof of Theorem~\ref{Big Main Theorem}.

\begin{proof}[Proof of Lemma~\ref{Lem:Balance}]
	By Markov's inequality and an elementary identity for the MGF of a natural exponential family,  
	\begin{eqnarray}
	Q_j(E_j) = \Pr[e^{\tau(\theta_1 - \theta_j)\overline{X}_1} > c] &\le& \frac{1}{c}\Exp[e^{\tau(\theta_1 - \theta_j)\overline{X}_1}] =  \frac{1}{c}\exp(\tau (A(2\theta_1 - \theta_j ) - A(\theta_1)) 
	\end{eqnarray}
	In particular, if we set $c_j = \frac{1}{\kappa} \exp(\tau (A(2\theta_1 -\theta_j) - A(\theta_1))$ then the above expression is no more than $\kappa$. With this choice of $c$,
	\begin{eqnarray*}
	\Pr(E_j) &=& \frac{1}{c}e^{\tau\{A(\theta_1) - A(\theta_j)\}}(1-Q_j(E_j))\\
	&=& \kappa e^{\tau\{2A(\theta_1) - A(\theta_j) - A(2\theta_1 -  \theta_j )\}}(1-Q_j(E_j))\\
	&\ge& \kappa(1-\kappa) e^{\tau\{2A(\theta_1) - A(\theta_j) - A(2\theta_1 - \theta_j)\}}
	\end{eqnarray*}
	We now invoke a well known property of exponential families
	\begin{fact}[\cite{nielsen2009statistical}]
	Let $\{p_{\theta}\}_{\theta \in \Theta}$ be an exponential family. Then for  $\theta,\widetilde{\theta} \in \Theta$, then $\kl(\theta,\widetilde{\theta}) = (\theta - \widetilde{\theta})A'(\theta) - A(\theta) + A(\widetilde{\theta})$, where $A'(\theta) = \int x p_{\theta}(x)d\nu(x)$ provided the integral exists.
	\end{fact}
	For ease of notation, set $d_j = \theta_1 - \theta_j$.  Then, 
	\begin{eqnarray*}
	2A(\theta_1) - A(\theta_j) - A(2\theta_1 - \theta_j) &=& A(\theta_1) - A(\theta_1 - d_j) + A(\theta_1) - A(\theta_1 + d_j) \\
	&=& A(\theta_1) - A(\theta_1 - d_j) - A'(\theta_1)d_j + A(\theta_1) - A(\theta_1 + d_j) + A'(\theta_1) d_j\\
	&=& - \kl(\theta_1,\theta_1 - d_j) - \kl(\theta_1,\theta_1 + d_j)
	\end{eqnarray*}
\end{proof}

%% file: distinct_measures_2.tex
\subsection{Deferred Proofs for Theorem~\ref{Big Main Theorem}}
	
\subsubsection{Proof of Fact~\ref{TVfact}\label{TVfactsection}}
	Let $\widehat{\Trs}$ with samples $\widehat{X}_{[a,s]}$ and denote the transcript from $\Sim(\nu)$ and let $\widehat{\Trs}^{(j)}$ with samples $\widehat{X}^{(j)}_{[a,s]}$ denote the transcript from $\Sim(\nu^{(j)})$. 

	First, note that under $M_j$ and $M_j^{(j)}$, all samples $\widehat{X}_{[a,s]}$ and $\widehat{X}^{(j)}_{[a,s]}$ for $a \in \{1,j\}$ and $s > \tau$ are i.i.d from $\nu_j$. Moreover, by symmetry of the construction under swapping the labels of $1$ and $j$, its easy to see that the samples $\widehat{X}_{[a,s]}$ and $\widehat{X}^{(j)}_{[a,s]}$ for $a \notin \{1,j\}$ have the same distribution under $M_j$ and $M_j^{(j)}$ respectively as well (even though  these samples are not necessarily going to be i.i.d from $\nu_a$ because of the conditioning). Hence,
	\begin{multline}
	\TV\left[ \Sim(\nu) \big{|} M_j; \Sim(\nu^{(j)}) \big{|} M_j^{(j)}\right] \\
	= \TV\left(\{\widehat{X}_{[1,s]},\widehat{X}_{[j,s]}\}_{1 \le s \le \tau} \big{|} M_j; \{\widehat{X}^{(j)}_{[1,s]},\widehat{X}^{(j)}_{[j,s]}\}_{1 \le s \le \tau}) | M_j^{(j)}\right)
	\end{multline}
	Since $\Sim$ doesn't actually change the first $\tau$ samples, we can actually drop this ${X}_{[a,s]}$ notation and just use $X_{[a,s]}$. 
	Next note that, $M_j$ is independent of $\{{X}_{[1,s]}\}_{1\le s\le \tau}$ and $M_j^{(j)}$ is independent of $\{{X}_{[j,s]}\}_{1\le s\le \tau}$. Hence, the first $\tau$ samples from arm $1$ (resp arm $j$) are i.i.d from $\nu_1$, and independent from the samples $\{{X}_{[j,s]}\}_{1 \le s \le \tau}$ (resp. $\{{X}_{[1,s]}\}_{1 \le s \le \tau}$ ). Using the $\TV$ bound $\TV(P_1 \otimes Q_1;P_2 \otimes Q_2) \le \TV(P_1;P_2) + \TV(Q_1;Q_2)$, for product measures $P_i\otimes Q_i$, we find that
	\begin{multline}
	\TV\left[ \Sim(\nu) \big{|} M_j; \Sim(\nu^{(j)}) \big{|} M_j^{(j)}\right] \\
	\le \TV\left(\{{X}_{[1,s]}\}_{1 \le s \le \tau}; \{{X}^{(j)}_{[1,s]}\}_{1 \le s \le \tau} | M_j^{(j)}\right) \\
	+ \TV\left(\{{X}_{[j,s]}\}_{1 \le s \le \tau} | M_j; \{{X}^{(j)}_{[j,s]}\}_{1 \le s \le \tau} \right) 
	\end{multline}
	By symmetry of construction, and symmetry of $\TV$ distance, its easy to check that 
	\begin{eqnarray}
	\TV\left(\{{X}_{[1,s]}\}_{1 \le s \le \tau}; \{{X}^{(j)}_{[1,s]}\}_{1 \le s \le \tau} | M_j^{(j)}\right) &=&  \TV\left(\{{X}^{(j)}_{[1,s]}\}_{1 \le s \le \tau} | M_j^{(j)}; \{{X}_{[1,s]}\}_{1 \le s \le \tau} \right) \\
	&=&\TV\left(\{{X}_{[j,s]}\}_{1 \le s \le \tau} | M_j; \{{X}^{(j)}_{[j,s]}\}_{1 \le s \le \tau} \right) \\
	&=& \TV\left(\{{X}_{[1,s]}\}_{1 \le s \le \tau}; \{{X}_{[j,s]}\}_{1 \le s \le \tau} | M_j \right) 
	\end{eqnarray}
	Hence, it suffices to check
	\begin{eqnarray}
	TV\left(\{{X}_{[1,s]}\}_{1 \le s \le \tau}; \{{X}^{(j)}_{[1,s]}\}_{1 \le s \le \tau} | M_j\right) = \TV(\overline{X}_1; \overline{X}_j | E_j )
	\end{eqnarray}
	We first use a sufficient statistic argument to reduce the total variation from samples to a $\TV$ between empirical means:
	\begin{claim}\label{SufficienStatisticClaim}
	\begin{eqnarray}
	TV\left(\{{X}_{[1,s]}\}_{1 \le s \le \tau}; \{{X}^{(j)}_{[1,s]}\}_{1 \le s \le \tau} | M_j\right) = \TV(\overline{X}_1; \overline{X}_j | M_j )
	\end{eqnarray}
	\end{claim}
	The proof is somewhat pedantic, and so we prove in just a moment. To conclude, we finally is to note that $\overline{X}_j  | M_j$ has the same distribution as $\overline{X}_j | E_j$, since
	\begin{eqnarray*}
	\Pr(\overline{X}_j  \in A | M_j ) &=& \Pr(\overline{X}_j  \in A \cap M_j ) /\Pr(M_j)\\
	&\overset{i}{=}& \Pr(\overline{X}_j  \in A \cap E_j,M_j ) /\Pr(M_j)\\
	&=& \Pr(E_j)\Pr(\overline{X}_j  \in A,M_j | E_j )(\Pr(E_j ) /\Pr(M_j)\\
	&\overset{ii}{=}& \Pr(\overline{X}_j  \in A | E_j)\Pr(M_j|E_j)\Pr(E_j)(\Pr(E_j \cap M_)) /\Pr(M_j)\\
	&=& \Pr(\overline{X}_j  \in A | E_j) /\Pr(M_j)\\
	&=& \Pr(\overline{X}_j  \in A | E_j ) 
	\end{eqnarray*}
	Where $i$ follows since $M_j \implies E_j$, and $ii$ follows since $M_j$ and $\overline{X}_j$ are conditionally independent given $E_j$.

	\begin{proof}[Proof of Claim~\ref{SufficienStatisticClaim}]
		Define the laws $P_1,P_j$ over the $(X_1,\dots,X_{\tau}) \in \R^{\tau}$  where under $P_1$, $(X_1,\dots,X_{\tau})$ have the law of ${X}_{[1,1]},\dots,{X}_{[1,\tau]}$, and under $P_j$, they have the law the law of ${X}_{[j,1]},\dots,{X}_{[j,\tau]}$. We use $P_j(|M_j)$ to denote the law of  ${X}_{[j,1]},\dots,{X}_{[j,\tau]}$ under $M_j$. Since $ {X}_{[j,1]},\dots,{X}_{[j,\tau]}$ are independent of $M_j$ given $\overline{X}_j$ (recall that $M_j$ depends only on some internal randomness and $E_j$, which depends only on $\overline{X}_j$). Hence, letting $\overline{X} = \sum_{s = 1}^{\tau} X_s$
		\begin{multline}\label{conditioningonm}
		P_j((X_1,\dots,X_{\tau}) = (x_1,\dots,x_{\tau}) | M_j) \\
		= P_j((X_1,\dots,X_{\tau}) = (x_1,\dots,x_{\tau}) | \overline{X} = \bar{x} )P_j(\overline{X} = \bar{x} | M_j)
		\end{multline}
		Moreover, since that since $\nu_1,\nu_j$ come from a one-parameter exponential family, 
		\begin{eqnarray}\label{conditioning2}
		P_1(\cdot | \overline{X} = \bar{x}) = P_j(\cdot | \overline{X} = \bar{x})
		\end{eqnarray}
		Thus, we conclude that
		\begin{eqnarray*}
		\TV(P_1;P_j | M_j) &=& \int_{\mathbf{x} \in \R^{\tau}} |dP_1(\mathbf{x}) - dP_j(\mathbf{x}|M_j)|\\
		&=& \int_{\bar{x}}\int_{\mathbf{x} : \sum_s \mathbf{x}_s = \tau\bar{x}} |dP_1(\bar{x}) dP_1(\mathbf{x} | \bar{x}) - dP_j(\bar{x},M_j) dP_j(\mathbf{x}|\bar{x},M_j)|\\
		&\overset{i}{=}& \int_{\bar{x}}\int_{\mathbf{x} : \sum_s \mathbf{x}_s = \tau\bar{x}} |dP_1(\bar{x}) dP_1(\mathbf{x} | \bar{x}) - dP_j(\bar{x},M_j) dP_j(\mathbf{x}|\bar{x})|\\
		&\overset{ii}{=}& \int_{\bar{x}}\int_{\mathbf{x} : \sum_s \mathbf{x}_s = \tau\bar{x}} |dP_1(\bar{x}) dP_1(\mathbf{x} | \bar{x}) - dP_j(\bar{x},M_j) dP_1(\mathbf{x}|\bar{x})|\\
		&=& \int_{\bar{x}} \int_{\mathbf{x} : \sum_s \mathbf{x}_s = \tau\bar{x}} dP_1(\mathbf{x} | \bar{x}) |dP_1(\bar{x})  - dP_j(\bar{x},M_j) |\\
		&\overset{iii}{=}& \int_{\bar{x}} |dP_1(\bar{x})  - dP_j(\bar{x},M_j) |(\int_{\mathbf{x} : \sum_s \mathbf{x}_s = \tau\bar{x}} dP_1(\mathbf{x} | \bar{x}) )\\
		&=& \int_{\bar{x}} |dP_1(\bar{x})  - dP_j(\bar{x},M_j)|\\
		&=& \TV(\overline{X}_1; \overline{X}_j | M_j)
		\end{eqnarray*}
		where $i$ follows from Equation~\ref{conditioningonm}, $ii$ follows from Equation~\ref{conditioning2}, $iii$ is Fubini's theorem.
	\end{proof}

\subsubsection{Proof of Conditional Simulated Le Cam(Lemma~\ref{ConditionalLeCam}) \label{Conditional-LeCam-Proof}}
	\begin{proof}[Proof of Lemma~\ref{ConditionalLeCam}]
	\begin{multline}\label{eq:MultiPairwise}
	 \Pr_{\nu,\Alg}\Pr[W^c|M_j] \overset{i}{=} \frac{1}{2}(\Pr_{\nu,\Alg}\Pr[W^c|M_j] + \Pr_{\nu^{(j)},\Alg}\Pr[W_j^c|M_j^{(j)}])  \\
	  \overset{ii}{\ge} \sup_{A \in \calF_T}\left|\Pr_{\nu,\Alg}\Pr[A|M_j] - \Pr_{\nu^{(j)},\Alg}\Pr[A|M_j^{(j)}]\right| - \TV[\Sim(\nu| M_j),\Sim(\nu^{(j)}| M_j^{(j)})] \\
	 	\overset{iii}{\ge} 1 - 2\Pr_{\nu,\Alg}\Pr[\{\hat{y} \ne 1\}|E_j] - \TV[\Sim(\nu| M_j),\Sim(\nu^{(j)}| M_j^{(j)})] 
	\end{multline}

	where $i$ follows from symmetry, $ii$ follows from applying Proposition~\ref{Prop:SimLeCam} using the measures $\nu | M_j$ and $\nu^{(j)} \big{|} M_j^{(j)}$,  (this time, with $\TV$ instead of $\KL$), and $iii$ follows since 
	\begin{eqnarray*}
	\sup_{A \in \calF_T}\left|\Pr_{\nu,\Alg}\Pr[A|M_j] - \Pr_{\tilde{\nu}_j,\Alg}\Pr[A|M_j^{(j)}]\right| &\ge& \Pr_{\nu,\Alg}\Pr[\{\hat{y} = 1\}|M_j] - \Pr_{\nu^{(j)},\Alg}\Pr[\{\hat{y} \ne 1 \}|M_j^{(j)}]  \\
	&\ge& \Pr_{\nu,\Alg}\Pr[\{\hat{y} = 1\}|M_j] - \Pr_{\nu^{(j)},\Alg}\Pr[\{\hat{y} \ne j \}|M_j^{(j)}]  \\
	&=& 1 - \Pr_{\nu,\Alg}\Pr[\{\hat{y} \ne 1\}|M_j] - \Pr_{\nu^{(j)},\Alg}\Pr[\{\hat{y} \ne j \}|M_j^{(j)}]  \\
	&=& 1 - 2\Pr_{\nu,\Alg}\Pr[\{\hat{y} \ne 1\}|M_j] 
	\end{eqnarray*}
	where the last line is a consequence of Fact~\ref{WjFact}.
	\end{proof}
\subsubsection{Proof of Lemma~\ref{ConditionalTV}\label{ConditionalTVProof}}
	\begin{proof}[Proof of Lemma~\ref{ConditionalTV}]
	Let call $\calF$ denote the $\sigma$ algebra generated by $X$. For any measures $\Pr$ and $\Q$ over $\calF$, note that $\Pr[A] - \Q[A] = \Q[A^c] - \Pr[A^c]$. Hence, 
	\begin{eqnarray*}
	\TV(\Pr,\Q) &=& \sup_{A \in \calF}|\Pr[A] - \Q[A]| \\
	 &=& \sup_{A \in \calF} \max\{\Pr[A] - \Q[A], \Pr[A^c] - \Q[A^c]\} \\
	 &=& \sup_{A \in \calF} \Pr[A] - \Q[A]
	\end{eqnarray*}
	Now $\Q = \Pr[\cdot | B]$. Since any $A \in \calF$ can be written as $A = A_B \sqcup A_{B^c}$ here $A_B \subset B$ and $A_{B^c} \subset B$
	\begin{eqnarray*}
	\TV(\Pr,\Q) &=& \sup_{A \in \calF} \Pr[A] - \Q[A] \\
	&=& \sup_{A_B \cup A_{B^c} \in \calF} \Pr[A_B \cup A_{B^c}] - \Q[A \cup A_{B^c}] \\
	&=& \sup_{A_B \cup A_{B^c} \in \calF} \{\Pr[A_B] - \Q[A_B] + \Pr[A_{B^c}] - \Q[A_{B^c}]\}\\
	&=& \sup_{A_B \subset B \in \calF} \{\Pr[A_B] - \Q[A_B]\} + \sup_{A_{B^c} \subset B^c \in \calF} \{\Pr[A_{B^c}] - \Q[A_{B^c}]\}
	\end{eqnarray*}
	For any $A_B \subset B$, we see $\Q[A_B] = \Pr[A_B \cap B]/\Pr[B] = \Pr[A_B]/\Pr[B]$, so $\Pr[A_B] - \Q[A_B] = (1 - \Pr[B]^{-1})\Pr(A_B) \le 0$, and thus $\sup_{A_B \subset B \in \calF} \{\Pr[A_B] - \Q[A_B]\} = 0$, by taking $A_B = \emptyset$. On the other hand, for $A_{B^c} \subset B^c$, $\Q[A_{B^c}] = \Pr[A_{B^c}\cap B]/\Pr[B] = 0$, and thus,
	\begin{eqnarray} \sup_{A_{B^c} \subset B^c \in \calF} \{\Pr[A_{B^c}] - \Q[A_{B^c}]\} = \sup_{A_{B^c} \subset B^c \in \calF} \Pr[A_{B^c}] = \Pr[B^c]\end{eqnarray}

\end{proof}

%% file: proof_of_topk.tex
\section{Proof of Proposition~\ref{TopKGaussianProp} \label{Sec:AlgRestrictions}}

	We prove Proposition~\ref{TopKGaussianProp} by arguing via ``algorithmic restrictions''. The basic idea is that, if a lower bound holds for one $\MAB$ or $\TopK$ problem, then it should also hold for the ``simpler'' $\MAB$ or $\TopK$ problem which arises by ``removing'' some of the arms.

	Formally, let $\nu = (\nu_a)_{a \in A}$ is an instance with arms indexed by $a \in A$ (where $A$ is finite). For $B \subset A$, we define the restriction of $\nu$ to $B$, denoted $\nu_{|B}$, as the instance $(\nu_{b})_{b \in B}$, indexed by arms $b \in B$. We let $\bS_A$ and $\bS_B$ denote the groups of permutations on the elements of $A$ and $B$, respectively. Finally, given a subset $S \subset A$ of ``good arms'', recall that we say an algorithm $\Alg$ with decision rule $\widehat{S} \subset A$ is $\delta$ correct in identifying $S$ over $\bS_A(\nu)$ if $\Pr_{\pi(\nu),\Alg}[\widehat{S} = \pi(S)] \ge 1 - \delta$ for all $\pi \in \bS_A$. 

	\begin{lemma}[Lower Bounds from Restrictions]\label{restriction_bound_lemma} Let $\nu = (\nu_a)_{a \in A}$ be an instance, $B \subset A$, and fix $\delta > 0$ and $b \in B$. Suppose that any algorithm $\Alg_{|B}$ which is $\delta$-correct in identifying $S \cap B$ over $\bS_{B}(\nu_{|B})$ satisfies the lower bound
	\begin{eqnarray}
	\Exp_{\sigma \sim \bS_B}\Pr_{\sigma(\nu_{|B}),\Alg_{|B}}[N_{\sigma(b)}(T)] \ge \tau] \ge 1 - \eta
	\end{eqnarray}
	for some $\tau,\eta > 0$ (which may depend on $\nu$, $S$, $B$, $\delta$ and $b$). Then any algorithm $\Alg$ which is $\delta$-correct in identifying $S$ over $\bS_{A}(\nu)$ satisfies the analogous lower bound 
	\begin{eqnarray}
	\Exp_{\pi \sim \bS_A}\Pr_{\pi(\nu),\Alg}[N_{\pi(b)}(T)] \ge \tau] \ge 1 - \eta
	\end{eqnarray}
	for the same $\tau$ and $\eta$.
	\end{lemma}

	To see how this lemma implies the bound for $\TopK$, let $A = [n]$, and for $j \in [k]$ and $\ell \in [n]\setminus [k]$, define the sets $B_j = \{j\}\cup ([n] \setminus [k])$ and $B_{\ell} = \{\ell\} \cup ([n] \setminus [k])$. Finally, let $S = [k]$ denotes the top $k$ arms, and $\widetilde{S} = [n]\setminus [k]$ denote the bottom $[n-k]$. Then, any $\delta$-correct algorithm over $\bS_n(\nu)$ equivalently identifies $S$ and $\widetilde{S}$ with probability of error at most $\delta$. Moreover, $B_j \cap S = \{j\}$, and $B_{\ell} \cap \widetilde{S} = \{\ell\}$. Now apply Lemma~\ref{restriction_bound_lemma} using the $\MAB$ lower bounds from Proposition~\ref{BestArmProp} for a) the problem of identifying $\nu_j$ from permutations of $\nu_{|B_j}$ and b) the problem $\nu_{\ell}$ from permutations of $\nu_{|B_\ell}$.

\subsection{Proof of Lemma~\ref{restriction_bound_lemma}}
	\begin{proof} 
	Let $\Alg$ be be $\delta$-correct in identifying $S$ over $\bS_{A}(\nu)$. Without loss of generality, we may assume that $\Alg$ is symmetric over $\bS_A$ (Lemma~\ref{SymmetryLemma}). We will now construct an algorithm $\Alg_{|B}$ which ``inherits'' the correctness and complexity of $\Alg$.
	\begin{claim}\label{AlgBClaim} There exists a symmetric (over $\bS_B$) algorithm $\Alg_{|B}$ with decision rule $\widehat{S}_{|B} \subset B$ which satisfies, for all $b \in B$
	\begin{eqnarray}
	\Pr_{\nu_{|B},\Alg_{|B}}[N_{b}(T)] \ge \tau] &=& \Pr_{\nu,\Alg}[N_{b}(T)] \ge \tau] \quad \text{and} \label{redux_complex_line} \\
	 \Pr_{\nu_{|B},\Alg_{|B}}[\widehat{S}_{|B} = S \cap B] &\ge& \Pr_{\nu,\Alg}[\widehat{S} = S ] \label{redux_correc_line}
	\end{eqnarray}
	\end{claim}
	Assume the above claim. Since $\Alg_{|B}$ is symmetric and $\Alg$ is $\delta$-correct, Equation~\ref{redux_correc_line} implies that $\Alg_{|B}$ is $\delta$-correct over $\bS_{B}(\nu_{|B})$, since all $\sigma \in \bS_B$, 
	\begin{eqnarray*}
	\Pr_{\sigma(\nu_{|B}),\Alg_{|B}}[\widehat{S}_{|B} = \sigma(S \cap B)] \ge \tau] = \Pr_{\nu_{|B},\Alg_{|B}}[\hat{y} \cap B = S \cap B]] \ge \Pr_{\nu,\Alg}[\widehat{S} = S ] \ge 1 -\delta
	\end{eqnarray*}
	Thus, by symmety of $\Alg_{|B}$ and the assumption of the lemma, we find for the choice of $b \in B$ and $\delta > 0$, 
	\begin{eqnarray}
	\Pr_{\nu_{|B},\Alg_{|B}}[N_{b}(T)] \ge \tau] = \Exp_{\sigma \sim \bS_B}\Pr_{\sigma(\nu_{|B}),\Alg_{|B}}[N_{\sigma(b)}(T)] \ge \tau]  \ge 1 - \eta
	\end{eqnarray}
	And hence, by Equation~\ref{redux_complex_line} and symmetry of $\Alg$,
	\begin{eqnarray}
	\Exp_{\pi \sim \bS_A}\Pr_{\pi(\nu),\Alg}[N_{\pi(b)}(T)] \ge \tau] =  \Pr_{\nu,\Alg}[N_{b}(T)] \ge \tau] = \Pr_{\nu_{|B},\Alg_{|B}}[N_{b}(T)] \ge \tau] \ge 1 - \eta
	\end{eqnarray}
	which concludes the proof.
	\end{proof}

	To conclude, we just need to verify that we can construct $\Alg_{|B}$ as in Claim~\ref{AlgBClaim}. To do this, let $\Trs_{|B}$ be a transcript samples $(X_{[b,s]})_{b \in B,s \in \N}$. For $a \in A \setminus B$, simulate a transcript $\Trs_{A \setminus B}$ of samples $(\widetilde{X}_{[a,s]})$ where $\widetilde{X}_{[a,s]} \overset{iid}{\sim} \nu_b$. Finally, let $\overline{\Trs}$ be the transcript obtained by concatening $\Trs_{|B}$ with the simulated transcript $\Trs_{A \setminus B}$, i.e. $\overline{X}_{[b,s]} = X_{[b,s]}$ for $b \in B$, and $\overline{X}_{[a,s]} = \widetilde{X}_{[a,s]}$. Finally, let $\Alg_{|B}$ be algorithm obtained by running $\Alg$ on the transcript $\Trs_{|B}$, with decision rule $\widehat{S}_{|B} = \widehat{S} \cap B $ (where $\widehat{S}$ is the decision rule of $A$). 

	Since $\overline{\Trs}$ has the same distribution as a transcript from $\nu$ when $\Trs_{|B}$ is drawn from $\nu_{|B}$, we immediate see that 
	\begin{eqnarray}
	\Pr_{\nu_{|B},\Alg_{|B}}[N_{b}(T)] \ge \tau] &=& \Pr_{\nu,\Alg}[N_{b}(T)] \ge \tau] \quad \text{and} \\
	 \Pr_{\nu_{|B},\Alg_{|B}}[\hat{y} \cap B = S \cap B] &=&  \Pr_{\nu,\Alg}[\widehat{S} \cap B = S \cap B] \ge \Pr_{\nu,\Alg}[\widehat{S} = S ] 
	\end{eqnarray}
	which verifies Equations~\ref{redux_complex_line} and~\ref{redux_correc_line}. It's also easy to check that $\Alg_{|B}$ is symmetric, since permuting $\Trs_{|B}$ under a permutation $\sigma \in \bS_{B}$ amounts to permuting $\overline{\Trs}$ by a permutation $\pi \in \bS_{B}$ which fixes elements of $B \setminus A$. Hence, symmetryof $\Alg_{|B}$ follows from symmetry of $\Alg$.